\documentclass[11pt]{article}

\usepackage{amsmath}
\usepackage{amsfonts}
\usepackage{amssymb, amsthm, epsfig}
\usepackage{mathrsfs}
\usepackage{euscript}

\usepackage{graphicx}
\usepackage{verbatim}
\usepackage{color}
\usepackage{epstopdf}

\usepackage{mathtools}
\usepackage{float}
\usepackage{algorithm}
\usepackage[noend]{algorithmic}

\usepackage{xspace}
\usepackage{color}

\usepackage{wrapfig}
\usepackage{bbm}

\usepackage{thmtools} 
\usepackage{thm-restate}
\usepackage{booktabs}

\usepackage[numbers,sort&compress]{natbib}
\usepackage[colorlinks=true, linkcolor=blue, citecolor=blue]{hyperref}

\newcommand{\eps}{\varepsilon}

\newcommand{\dir}{\ensuremath{\mathsf{d}}}

\DeclareMathOperator{\poly}{\textsf{poly}}

\newcommand{\Var}{\ensuremath{\textsf{Var}}}

\newcommand{\norm}[1]{\lVert #1 \rVert}

\newcommand\numberthis{\addtocounter{equation}{1}\tag{\theequation}}
\newcommand{\abs}[1]{| #1 |}
\newcommand{\setdef}[2]{\left\{ #1 \mid #2 \right\}}

\newcommand{\E}{\ensuremath{\textsf{E}}}

\newtheorem{theorem}{Theorem}
\newtheorem{lemma}[theorem]{Lemma}

\usepackage[inline]{enumitem}

\usepackage[a4paper]{geometry}
\title{Learning Mixtures of Gaussians with Censored Data}
\author{Wai Ming Tai \\ University of Chicago \and Bryon Aragam \\ University of Chicago}

\date{}

\begin{document}

\maketitle
\begin{abstract}
We study the problem of learning mixtures of Gaussians with censored data. Statistical learning with censored data is a classical problem, with numerous practical applications, however, finite-sample guarantees for even simple latent variable models such as Gaussian mixtures are missing.
Formally, we are given censored data from a mixture of univariate Gaussians 
$$
    \sum_{i=1}^k w_i \mathcal{N}(\mu_i,\sigma^2),
$$
i.e. the sample is observed only if it lies inside a set $S$.
The goal is to learn the weights $w_i$ and the means $\mu_i$.
We propose an algorithm that takes only $\frac{1}{\varepsilon^{O(k)}}$ samples to estimate the weights $w_i$ and the means $\mu_i$ within $\varepsilon$ error.
\end{abstract}

\section{Introduction}

When we collect data, we often encounter situations in which the data are partially observed. 
This can arise for a variety of reasons, such as measurements falling outside of the range of some apparatus or device. 
In machine learning and statistics, this phenomenon is known as truncated or censored data. 
Both refer to the case where we do not observe the data when they fall outside a certain domain. 
For censored data, we know the existence of data that fall outside the domain, while for truncated data, we do not.

It is common to encounter truncated or censored data in our daily lives. 
An example of truncated data is census data. When a census bureau collects data, there may be some difficulties for the bureau in collecting data for certain demographics for security, privacy, or legal reasons,
and these individuals may have no incentive to report their data.
Therefore, the bureau cannot collect data about these populations.
In this case, the census data are truncated.

On the other hand, an example of censored data is test scores. 
The range of scores in a test is typically set to be from 0 to 100. Some students may score the maximum score of 100 points, in which case it is unknown if they could have scored even higher if the upper bound of the test score was higher than 100. 
Of course, even though the students' scores are capped at 100, their scores are still being reported. 
Hence, their scores are censored, which distinguishes them from truncated data.

Indeed, statistical estimation on truncated or censored data is a classical problem, dating back to the eighteenth century \cite{bernoulli1760essai}.
After Bernoulli, \cite{galton1898examination,pearson1902systematic,pearson1908generalised,lee1914table,fisher1931properties} studied how to estimate the mean and the variance of of a univariate Gaussian distribution from truncated samples.
However, most existing results do not address the problem of finite-sample bounds, i.e. the results are mostly experimental or asymptotic \cite{lee2012algorithms,mclachlan1988fitting}.
In fact, one can learn the distribution with infinitely many truncated or censored samples---under mild assumptions, one can show that the function restricted on a certain region can be extended to the entire space by the identity theorem from complex analysis.
Unfortunately, it is still not clear how to translate such results to finite sample bounds.

A recent notable result by \citet{daskalakis2018efficient} gave the first efficient algorithm to learn the mean and the covariance of a single Gaussian with finitely many truncated samples.
A natural extension to the problem of learning a single Gaussian is the problem of learning a mixture of Gaussians.
To the best of our knowledge, there is no provable guarantees on the problem of learning a mixture of Gassians with a finite number of truncated or censored samples even in one dimension.

As we will discuss in the related work section, there is a long line of work on learning a mixture of Gaussians.
Likelihood-based approaches often do not provide provable guarantees for learning mixtures of Gaussians since the objective function is not convex unless we impose strong assumptions \cite{xu2016global,daskalakis2017ten}.
On the other hand, many recent results rely heavily on the method of moments, i.e. the algorithm estimates the moments $\E(X^s)$ as an intermediate step.
With truncated or censored data, estimating $\E(X^s)$ (here, the expectation is over the original, untruncated data) becomes very challenging.

To overcome this, we propose an approach for estimating moments from censored data.
Recall that ordinary moments are just expectations of monomials of a random variable.
However, by generalizing this to more general functions of a random variable, we open up the possibility to capture more complex structures of the distribution. 
In particular, when the data is censored, these generalized functions allow us to relate the expectations back the raw, uncensored distribution.
One must keep in mind that we still need to make a choice of what functions to consider in addition to providing efficient estimators of these generalized moments. 
We preview that a suitable choice is found by a specific linear combination of Hermite polynomials derived from the solution to a system of linear equations.
In our proof, we will delve deeper into the analysis of the expectations of functions, depending on the domain, and provide a delicate analysis to prove our desired result.

Based on the above discussion, we may want to ask the following question in a general sense: \emph{Can we learn a mixture of Gaussians with truncated or censored data?}
In this paper, we consider this problem and focus on the case that the data is censored and the Gaussians are univariate and homogeneous.
We now define the problem formally.

\section{Problem Definition}

Let $\mathcal{N}(\mu,\sigma^2)$ be the normal distribution with mean $\mu$ and variance $\sigma^2$.
Namely, the pdf of $\mathcal{N}(\mu,\sigma^2)$ is 
\begin{align*}
    g_{\mu,\sigma^2}(x) := \frac{1}{\sqrt{2\pi}\sigma}e^{-\frac{1}{2\sigma^2}(x-\mu)^2}.
\end{align*}
For any subset $S\subset\mathbb{R}$, let $I_{\mu,\sigma^2}(S)$ be the probability mass of $\mathcal{N}(\mu,\sigma^2)$ on $S$, i.e.
\begin{align*}
    I_{\mu,\sigma^2}(S) := \int_{x\in S} g_{\mu,\sigma^2}(x) \dir x.
\end{align*}
Also, let $\mathcal{N}(\mu,\sigma^2,S)$ denote the conditional distribution of a normal $\mathcal{N}(\mu,\sigma^2)$ given the set $S$.
Namely, the pdf of $\mathcal{N}(\mu,\sigma^2,S)$ is 
\begin{align*}
    g_{\mu,\sigma^2,S}(x) := 
    \begin{cases}
        \frac{1}{I_{\mu,\sigma^2}(S)} g_{\mu,\sigma^2}(x) &\text{if $x\in S$} \\
        0 & \text{if $x\notin S$}.
    \end{cases}
\end{align*}

Given a subset $S\subset\mathbb{R}$,
we consider the following sampling procedure.
Each time, a sample is drawn from a mixture of Gaussians 
\begin{align*}
    & \sum_{i=1}^k w_i\mathcal{N}(\mu_i,\sigma^2),
\end{align*}
where $w_i>0$, $\sum_{i=1}^kw_i=1$, $\mu_i\in\mathbb{R}$, and $\sigma>0$.
If this sample is inside $S$, we obtain this sample; otherwise, we fail to generate a sample.
Formally, $X$ is a random variable drawn from the following distribution.
Let $\alpha$ be the probability mass $\sum_{i=1}^k w_i I_{\mu_i,\sigma^2}(S)$.
\begin{align*}
    X \sim
    \begin{cases}
        \sum_{i=1}^k w_i\mathcal{N}(\mu_i,\sigma^2,S) & \text{with probability $\alpha$} \\
        \textsf{FAIL} & \text{with probability $1-\alpha$}.
    \end{cases} \numberthis\label{eq:model}
\end{align*}
The value \textsf{FAIL} here refers to values that are not directly accessible to the algorithm.

We assume that 
\begin{itemize}
    \item[] (A1) $S$ is an interval $[-R,R]$ for some constant $R>0$ and is known (it is easy to extend $S$ to be any measurable subset of $[-R,R]$; for simplicity, we assume $S=[-R,R]$);
    \item[] (A2) All $\mu_i$ are bounded, i.e. $\abs{\mu_i}<M$ for some constant $M>0$;
    \item[] (A3) The variance $\sigma^2$ is known.
\end{itemize}
We also assume that the exact computation of the integral $\int_0^z e^{-\frac{1}{2}t^2}\dir t$ for any $z$ can be done.
Indeed, one can always approximate this integral with an exponential convergence rate via Taylor expansion.
As we can see in our proof, this error is negligible.

For a given error parameter $\eps>0$, we want to estimate all $w_i,\mu_i$ within $\eps$ error.
The question is \emph{how many samples from the above sampling procedure do we need to achieve this goal?}
Our main contribution is a quantitative answer to this question. We will prove the following theorem:

\begin{theorem}\label{thm:main}
Suppose we have $n$ samples drawn from the distribution \eqref{eq:model} and we assume that the mixture satisfies (A1)-(A3).
Furthermore, let $w_{\min}$ be $\min\setdef{w_i}{i=1,\dots,k}$ and $\Delta_{\min}$ be $\min\setdef{\abs{\mu_i-\mu_j}}{i,j=1,\dots,k\text{ and }i\neq j}$.
Then, for a sufficiently small $\eps>0$, if $w_{\min}$ and $\Delta_{\min}$ satisfy $w_{\min}\Delta_{\min} = \Omega(\eps)$, there is an efficient algorithm that takes $n=C_k \cdot \frac{1}{\eps^{O(k)}}$ (where $C_k$ is a constant depending on $k$ only) samples\footnote{Here we assume the parameters $R$ and $M$ to be constant for simplicity.
It is easy to keep track of them in our proof and show that the sample bound is $C_k\cdot (\frac{1}{\eps})^{O(k\cdot \log(M+R+\frac{1}{R}))}$.} as the input and outputs $\widehat{w}_i,\widehat{\mu}_i$ for $i=1,\dots,k$ such that, up to an index permutation $\Pi$,
\begin{align*}
    \abs{\widehat{w}_{\Pi(i)}-w_i}<\eps,\qquad \abs{\widehat{\mu}_{\Pi(i)}-\mu_i}<\eps\qquad \text{for $i=1,\dots,k$}
\end{align*}
with probability $\frac{99}{100}$.
The running time of the algorithm is $O(n\cdot\poly(k,\frac{1}{\eps}))$.
\end{theorem}

In other words, this theorem states that the sample complexity for learning mixtures of $k$ univariate Gaussians with censored data is $\frac{1}{\eps^{O(k)}}$ which is optimal in terms of asymptotic growth of the exponent $O(k)$ \cite{wu2018optimal}.
As for the optimality of the constant in the exponent $O(k)$, this is an interesting open problem.

\section{Related Work}

Without truncation or censoring, the study of learning Gaussian mixture models \cite{pearson1894contributions} has a long history. 
We focus on recent algorithmic results; see \citet{lindsay1995} for additional background.
\citet{dasgupta1999learning} proposed an algorithm to learn the centers of each Gaussian when the centers are $\Omega(\sqrt{d})$ apart from each other.
There are other results such as \cite{vempala2004spectral,regev2017learning} that are based on similar separation assumptions and that use clustering techniques.

There are other results using the method of moments.
Namely, the algorithm estimates the moments $\E(X^s)$ as an intermediate step.
\citet{moitra2010settling,kalai2010efficiently} showed that, assuming $k = O(1)$, there is an efficient algorithm that learns the parameters with $\frac{1}{\eps^{O(k)}}$ samples.
\citet{hardt2015tight} showed that, when $k = 2$, the optimal sample complexity of learning the parameters is $\Theta(\frac{1}{\eps^{12}})$.
For the case that the Gaussians in the mixture have equal variance, \citet{wu2018optimal} proved the optimal sample complexity for learning the centers is $\Theta(\frac{1}{\eps^{4k-2}})$ if the variance is known and $\Theta(\frac{1}{\eps^{4k}})$ if the variance is unknown.
Later, \citet{doss2020optimal} extended the optimal sample complexity to high dimensions.

When the data are truncated or censored, however, the task becomes more challenging.
\cite{schneider1986truncated,balakrishnan2014art,cohen2016truncated} provided a detailed survey on the topic of learning Gaussians with truncated or censored data.
Recently, \citet{daskalakis2018efficient} showed that, if the samples are from a single Gaussian in high dimensional spaces, 
there is an algorithm that uses $\widetilde{O}(\frac{d^2}{\eps^2})$ samples to learn the mean vector and the covariance matrix.
Their approach is likelihood based.
Namely, they optimize the negative log-likelihood function to find the optimal value.
This approach relies on the fact that, for a single Gaussian, the negative log-likelihood function is convex and hence one can use greedy approaches such as stochastic gradient descent to find the optimal value.

Unfortunately, when there are multiple Gaussians in the mixture, we may not have such convexity property for the negative log-likelihood function.
\citet{nagarajan2020analysis} showed that, for the special case of a truncated mixture of two Gaussians whose centers are symmetric around the origin and assuming the truncated density is known, the output by the EM algorithm converges to the true mean as the number of iterations tends to infinity.

There are other problem settings that are closely related to ours such as robust estimation of the parameters of a Gaussian in high dimensional spaces.
The setting of robust estimation is the following.
The samples we observed are generated from a single high dimensional Gaussians except that a fraction of them is corrupted.
Multiple previous results such as  \cite{hopkins2022efficient,liu2021robust,diakonikolas2019recent,diakonikolas2019robust,lai2016agnostic,diakonikolas2017being,diakonikolas2018robustly} proposed learning algorithms to learn the mean vector and the covariance matrix.

Regression with truncated or censored data is another common formulation.
Namely, we only observe the data when the value of the dependent variable lies in a certain subset.
A classic formulation is the truncated linear regression model \cite{tobin1958estimation,amemiya1973regression,hausman1977social,maddala1986limited}.
Recently, in the truncated linear regression model, \citet{daskalakis2019computationally} proposed a likelihood-based estimator to learn the parameters.

\section{Preliminaries}

We denote the set $\{0,1,\dots,n-1\}$ to be $[n]$ for any positive integer $n$.
Let $h_j(x)$ be the (probabilist's) Hermite polynomials, i.e.
\begin{align*}
	h_j(x) & = (-1)^je^{\frac{1}{2}x^2}\frac{\dir^j}{\dir \xi^j}e^{-\frac{1}{2}\xi^2}\bigg|_{\xi=x} \qquad\text{for all $x\in\mathbb{R}$.}
\end{align*}
Hermite polynomials can also be given by the exponential generating function, i.e.
\begin{align*}
    e^{x\mu-\frac{1}{2}\mu^2} = \sum_{j=0}^\infty h_j(x)\frac{\mu^j}{j!}\qquad \text{for any $x,\mu\in\mathbb{R}$.} \numberthis\label{eq:exp_hermite}
\end{align*}
Also, the explicit formula for $h_j$ is 
\begin{align*}
    h_j(x) = j!\sum_{i=0}^{\lfloor j/2\rfloor} \frac{(-1/2)^i}{i!(j-2i)!} x^{j-2i}
\end{align*}
and this explicit formula is useful in our analysis.

In our proof, we will solve multiple systems of linear equations. 
Cramer's rule provides an explicit formula for the solution of a system of linear equations whenever the system has a unique solution.
\begin{lemma}[Cramer's rule]\label{lem:cramer}
	
	Consider the following system of $n$ linear equations with $n$ variables.
	\begin{align*}
		Ax=b
	\end{align*}
	where $A$ is a $n$-by-$n$ matrix with nonzero determinant and $b$ is a $n$ dimensional vector.
	Then, the solution of this system $\widehat x = A^{-1}b$ satisfies that the $i$-th entry of $\widehat x$ is 
	\begin{align*}
		\det (A^{(i\leftarrow b)}) / \det (A)
	\end{align*}
	where $A^{(i\leftarrow b)}$ is the same matrix as $A$ except that the $i$-th column is replaced with $b$.
	
\end{lemma}

Thanks to the application of Cramer's rule, we often encounter determinants.
The Cauchy-Binet formula is a formula for the determinant of a matrix that each entry can be expressed as an inner product of two vectors that correspond to its row and column.
Note that the Cauchy-Binet formula usually applies to the case that the entries are finite sums.
For our purpose, we state the Cauchy-Binet formula for the case that the entries are in integral form.
\begin{lemma}[Cauchy–Binet formula]\label{lem:cb_formula}
	Let $A$ be a $n$-by-$n$ matrix whose $(r,c)$-entry has a form of $\int_{x\in S} f_r(x)g_c(x)\dir x$ for some functions $f_r,g_c$ and some domain $S\subset\mathbb{R}$.
	Then, the determinant of $A$ is 
	\begin{align*}
	    \det(A) = \int_{x_0>\cdots>x_{n-1}, \mathbf{x}\in S^n} \det(B(\mathbf{x}))\cdot \det(C(\mathbf{x})) \dir \mathbf{x} 
	\end{align*}
	where, for any $\mathbf{x}=(x_0,\dots,x_{n-1})\in S^n$, $B(\mathbf{x})$ is a $n$-by-$n$ matrix whose $(r,i)$-entry is $f_r(x_i)$ and $C(\mathbf{x})$ is a $n$-by-$n$ matrix whose $(i,c)$-entry is $g_c(x_i)$.
\end{lemma}

Another tool to help us compute the determinants is Schur polynomials.
Schur polynomials are defined as follows.
For any partition $\lambda=(\lambda_1,\dots,\lambda_n)$ such that $\lambda_1\geq \cdots\geq \lambda_n$ and $\lambda_i\geq 0$, define the function $a_{(\lambda_1+n-1, \lambda_2+n-2,\dots,\lambda_n)}(x_1,x_2,\dots,x_n)$ to be 
\begin{align*}
    \MoveEqLeft a_{(\lambda_1+n-1, \lambda_2+n-2,\dots,\lambda_n)}(x_1,x_2,\dots,x_n)  \\
    & :=
    \det \begin{bmatrix}
    x_1^{\lambda_1+n-1} & x_2^{\lambda_1+n-1} & \cdots & x_n^{\lambda_1+n-1} \\
    x_1^{\lambda_2+n-2} & x_2^{\lambda_2+n-2} & \cdots & x_n^{\lambda_2+n-2} \\
    \vdots & \vdots & \ddots & \vdots \\
    x_1^{\lambda_n} & x_2^{\lambda_n} & \cdots & x_n^{\lambda_n} 
    \end{bmatrix}.
\end{align*}
In particular, when $\lambda=(0,0,\dots,0)$, it becomes the Vandermonde determinant, i.e.
\begin{align*}
    a_{(n-1, n-2,\dots,0)}(x_1,x_2,\dots,x_n) 
    & =
    \prod_{1\leq j<k\leq n}(x_j-x_k).
\end{align*}
Then, Schur polynomials are defined to be 
\begin{align*}
    s_\lambda(x_1,x_2,\dots,x_n)  
    & := \frac{a_{(\lambda_1+n-1, \lambda_2+n-2,\dots,\lambda_n)}(x_1,x_2,\dots,x_n)}{a_{(n-1, n-2,\dots,0)}(x_1,x_2,\dots,x_n) }.
\end{align*}
It is known that $s_\lambda(x_1,x_2,\dots,x_n)$ can be written as $\sum_Y \mathbf{x}^Y$ where the summation is over all semi-standard Young tableaux $Y$ of shape $\lambda$.
Here, each term $\mathbf{x}^Y$ means $x_1^{y_1}\cdots x_{n}^{y_{n}}$ where $y_i$ is the number of occurrences of the number $i$ in $Y$ and note that $\sum_{i=1}^{n}y_i = \sum_{i=1}^n\lambda_i$.
Also, a semi-standard Young tableau $Y$ of shape $\lambda=(\lambda_1,\dots,\lambda_n)$ can be represented by a finite collection of boxes arranged in left-justified rows where the row length is $\lambda_i$ and each box is filled with a number from $1$ to $n$ such that the numbers in each row is non-decreasing and the numbers in each column is increasing.
To avoid overcomplicating our argument, when we count the number of semi-standard Young tableaux of some shape we only use a loose bound for it.

\section{Proof Overview}

Recall that our setting is the following (cf. \eqref{eq:model}):
We are given samples drawn from the following sampling procedure.
Each time, a sample is drawn from a mixture of Gaussians 
\begin{align*}
    & \sum_{i=1}^k w_i\mathcal{N}(\mu_i,\sigma^2)
\end{align*}
where $w_i>0, \sum_{i=1}^kw_i=1, \mu_i\in\mathbb{R}$ and $\sigma>0$.
If this sample is inside $S$, we obtain this sample; otherwise, we fail to generate a sample.
Our goal is to learn $w_i$ and $\mu_i$.

One useful way to view mixtures of Gaussians is to express it as 
\begin{align*}
    \bigg(\sum_{i=1}^k w_i\delta_{\mu_i}\bigg) * \mathcal{N}(0,\sigma^2)
\end{align*}
where $\delta_{\mu_i}$ is the delta distribution at $\mu_i$ and $*$ is the convolution operator.
We call the distribution $\sum_{i=1}^k w_i\delta_{\mu_i}$ the mixing distribution.
Let $\mathbf{m}_j$ be the moment of the mixing distribution, i.e.
\begin{align*}
    \mathbf{m}_j := \sum_{i=1}^k w_i\mu_i^j.
\end{align*}
Since we assume that the variance is known, without loss of generality, we set $\sigma=1$; otherwise, we can scale all samples such that $\sigma=1$.
First, we reduce the problem to estimating $\mathbf{m}_j$, so that we can employ known results on estimating mixtures of Gaussians using the method of moments. For example, \citet{wu2018optimal} proved the following theorem.
\begin{theorem}[Denoised method of moments, \cite{wu2018optimal}]\label{thm:denoised}
Suppose $\mathbf{m}_j$ are the moments of a distribution that has $k$ supports on $\mathbb{R}$, i.e. $\mathbf{m}_j$ has a form of $\sum_{i=1}^kw_i\mu_i^j$ where $w_i>0$, $\sum_{i=1}^kw_i=1$ and $\mu_i\in\mathbb{R}$. 
Let $w_{\min}$ be $\min\setdef{w_i}{i=1,\dots,k}$ and $\Delta_{\min}$ and $\min\setdef{\abs{\mu_i-\mu_j}}{i,j=1,\dots,k\text{ and }i\neq j}$.
For any $\delta>0$, let $\widehat{\mathbf{m}}_j$ be the numbers that satisfy
\begin{align*}
    \abs{\widehat{\mathbf{m}}_j - \mathbf{m}_j}<\delta \qquad \text{for all $j=1,\dots,2k-1$.}
\end{align*}
Then, if $w_{\min}$ and $\Delta_{\min}$ satisfy $w_{\min}\Delta_{\min} = \Omega(\delta^{O(\frac{1}{k})})$, there is an algorithm that takes $\widehat{\mathbf{m}}_j$ as the input and outputs $\widehat{w}_i,\widehat{\mu}_i$ such that, up to an index permutation $\Pi$,
\begin{align*}
    \abs{\widehat{w}_{\Pi(i)} - w_i} < C_k\cdot\frac{\delta^{\Omega(\frac{1}{k})}}{w_{\min}}
\end{align*}
and
\begin{align*}
    \abs{\widehat{\mu}_{\Pi(i)} - \mu_i} < C_k\cdot\frac{\delta^{\Omega(\frac{1}{k})}}{\Delta_{\min}}
\end{align*}
where $C_k$ is a constant depending on $k$ only.

\end{theorem}

Unfortunately, unlike with fully observed mixtures of Gaussians, estimating these moments is no longer straightforward.
As we will see in our proof, looking for unbiased estimators relies on specific structures of Gaussians.
When the data is censored, such structures may not exist.
Hence, we look for a biased estimator and provide delicate analysis to bound the bias.
To see how we can estimate $\mathbf{m}_j$, we first express the mixture as an expression that is in terms of $\mathbf{m}_j$.
Suppose $X$ is the random variable drawn from the sampling procedure conditioned on non-\textsf{FAIL} samples.
For any function $f$, the expectation of $f(X)$ is 
\begin{align*}
    \E(f(X)) 
    & =
    \int_{S} f(x) \cdot \left(\sum_{i=1}^k w_i g_{\mu_i,1,S}(x)\right)\dir x  
    =
    \frac{1}{\alpha}\cdot \int_{S} f(x) \cdot \left(\sum_{i=1}^k w_i g_{\mu_i,1}(x)\right)\dir x.\numberthis\label{eq:exp_f}
\end{align*}
Recall that $\alpha$ is the probability mass $\sum_{i=1}^k w_iI_{\mu_i,1}(S)$.
Note that, for any $\mu$,
\begin{align*}
    g_{\mu,1}(x) &= \frac{1}{\sqrt{2\pi}}e^{-\frac{1}{2}(x-\mu)^2} 
     = \frac{1}{\sqrt{2\pi}}e^{-\frac{1}{2}x^2}e^{x\mu-\frac{1}{2}\mu^2} 
     =
    \frac{1}{\sqrt{2\pi}}e^{-\frac{1}{2}x^2}\sum_{j=0}^\infty h_{j}(x)\frac{\mu^j}{j!}\numberthis\label{eq:g_hermite}
\end{align*}
where $h_j$ is the $j$-th Hermite polynomial and the last equality is from the fact \eqref{eq:exp_hermite}.
In other words, when we plug \eqref{eq:g_hermite} into \eqref{eq:exp_f}, we have
\begin{align*}
    \alpha\cdot\E(f(X)) 
    & =
    \int_{S} f(x) \cdot \left(\sum_{i=1}^k w_i\cdot\left(\frac{1}{\sqrt{2\pi}}e^{-\frac{1}{2}x^2}\sum_{j=0}^\infty h_{j}(x)\frac{\mu_i^j}{j!}\right)\right)\dir x \\
    & =
    \sum_{j=0}^\infty \left(\int_{S} f(x)\cdot \frac{1}{\sqrt{2\pi}j!}e^{-\frac{1}{2}x^2}h_j(x) \dir x\right)\cdot\left(\sum_{i=1}^kw_i\mu_i^j\right).
\end{align*}
To ease the notation, for any function $f$ and positive integer $j$, we define
\begin{align*}
    J_{f,j}:=\int_{S} f(x)\cdot \frac{1}{\sqrt{2\pi}j!}e^{-\frac{1}{2}x^2}h_j(x) \dir x.\numberthis\label{eq:j_def}
\end{align*}
If we plug $J_{f,j}$ and $\mathbf{m}_j$ into the equation for $\alpha\cdot\E(f(X))$, we have
\begin{align*}
    \alpha\cdot \E(f(X)) = \sum_{j=0}^\infty J_{f,j}\cdot \mathbf{m}_j. \numberthis\label{eq:expf}
\end{align*}
Ideally, \emph{if} we manage to find $2k-1$ functions $f_1,\dots,f_{2k-1}$ such that
\begin{align*}
    J_{f_i,j} = \begin{cases}
        1 \qquad \text{if $i=j$} \\
        0 \qquad \text{if $i\neq j$.}
    \end{cases}
\end{align*}
then we have
\begin{align*}
    \alpha\cdot \E(f_i(X)) = \mathbf{m}_i \qquad\text{for all $i=1,\dots,2k-1$}.
\end{align*}
It means that we will have an unbiased estimator for $\mathbf{m}_i$ and therefore we just need to find out the amount of samples we need by bounding the variance.
Indeed, if $S=\mathbb{R}$ and we pick $f_i$ to be the $i$-th Hermite polynomial $h_i$ then the aforementioned conditions hold.
It is how \cite{wu2018optimal} managed to show their result.
However, when $S\neq \mathbb{R}$, it becomes trickier.

A natural extension is to pick $f_i$ to be a linear combination of Hermite polynomials, i.e.
\begin{align*}
    f_i = \sum_{a=0}^{\ell-1} \beta_{i,a} h_{a}
\end{align*}
for some positive integer $\ell$.
The integer $\ell$ is a parameter indicating how accurate our estimator is.
Indeed, this $\ell\rightarrow \infty$ as $\eps\rightarrow 0$ as we will show in our proof.
For each $f_i$, there are $\ell$ coefficients $\beta_{i,j}$ in the expression and therefore we can enforce $\ell$ terms  of $J_{f_i,j}$ to be the desired values.
More precisely, we can set $\beta_{i,a}$ such that
\begin{align*}
    \MoveEqLeft J_{f_i,j} 
    = 
    \int_{S} f_i(x)\cdot \frac{1}{\sqrt{2\pi}j!}e^{-\frac{1}{2}x^2}h_j(x) \dir x\\
    & =
    \sum_{a=0}^{\ell-1} \beta_{i,j}\int_{S} h_a(x)\cdot \frac{1}{\sqrt{2\pi}j!}e^{-\frac{1}{2}x^2}h_j(x) \dir x\\
    & =
    \sum_{a=0}^{\ell-1} \beta_{i,a}J_{h_a,j} 
    =
    \begin{cases}
        1 \qquad \text{if $i=j$}\\
        0 \qquad \text{if $i\neq j$}
    \end{cases}
\end{align*}
for $j=0,\dots,\ell-1$.
If we assume the integrals can be computed exactly, then all $J_{h_a,j}$ are known.
Hence, we can solve $\beta_{i,a}$ by solving this system of linear equations.

Now, if we plug them into \eqref{eq:expf} then we have
\begin{align*}
    \alpha \cdot \E(f_i(X))= \mathbf{m}_i + \underbrace{\sum_{j=\ell}^\infty J_{f_i,j}\cdot \mathbf{m}_j}_{:=\mathcal{E}_i}.
\end{align*}
Note that the term $\mathbf{m}_i$ is what we aim at and hence the term $\mathcal{E}_i$ is the error term.
Indeed, our estimator is a biased estimator where the bias is $\mathcal{E}_i$.
Thanks to the factor $\frac{1}{j!}$ in the term $J_{f_i,j}$, intuitively, the term $\mathcal{E}_i \rightarrow 0$ as $\ell \rightarrow 0$.

Define our estimator $\widehat{\mathbf{m}}_i$ to be 
\begin{align*}
    \widehat{\mathbf{m}}_i=\frac{1}{n}\left(\sum_{s=1}^{n'}f_i(x_s)\right) \numberthis\label{eq:estimator}
\end{align*}
where $n'$ is the number of samples that are non-\textsf{FAIL} and $x_i$ are the non-\textsf{FAIL} samples.
Note that, on average, the term $\frac{1}{n} = \frac{\alpha}{n'}$ gives us the factor $\alpha$ implicitly.
Then, by Chebyshev's inequality, we have
\begin{align*}
    \abs{\widehat{\mathbf{m}}_i - \mathbf{m}_i} < \delta + \abs{\mathcal{E}_i} \quad \text{with probability $1-\frac{\Var(\widehat{\mathbf{m}}_i)}{\delta^2}$.}
\end{align*}
Now, we break the problem down to the following two subproblems.
\begin{itemize}
    \item How large $\ell$ needs to be in order to make $\abs{\mathcal{E}_i} < \delta$?
    \item Given $\delta>0$, how many samples do we need to make the variance $\Var(\widehat{\mathbf{m}}_i) < \frac{\delta^2}{100}$ and hence the success probability larger than $\frac{99}{100}$?
\end{itemize}
Detailed proofs are deferred to the appendix.

\subsection{Bounds for the Number of Terms}\label{sec:how_large_ell}

To see how large $\ell$ needs to be, we first define the following notations.
Let $v^{(j)}$ be the $\ell$-dimensional vector whose $a$-th entry is $J_{h_a,j}$, i.e.
\begin{align*}
    v^{(j)} = 
    \begin{bmatrix}
    J_{h_0,j} &
    J_{h_1,j} &
    \cdots &
    J_{h_{\ell-1},j}
    \end{bmatrix}^\top,
\end{align*}
and $V$ be the the $\ell$-by-$\ell$ matrix whose $r$-th row is $(v^{(r)})^\top$, i.e.
\begin{align*}
    V =
    \begin{bmatrix}
    v^{(0)} &
    v^{(1)} &
    \cdots &
    v^{(\ell-1)}
    \end{bmatrix}^\top. \numberthis \label{eq:v_def}
\end{align*}
Recall that, by the definition of $\beta_{i,a}$, $\beta_{i,a}$ satisfies 
\begin{align*}
    \sum_{a=0}^{\ell-1} \beta_{i,a}J_{h_a,j} =
    \begin{cases}
        1 \qquad \text{if $i=j$}\\
        0 \qquad \text{if $i\neq j$}
    \end{cases}.
\end{align*}
We can rewrite it as a system of linear equations.
\begin{align*}
    V\beta_i=\mathbf{e}_i \numberthis\label{eq:sys_eq}
\end{align*}
where $\beta_i$ is the $\ell$-dimensional vector whose $a$-th entry is $\beta_{i,a}$ and $\mathbf{e}_i$ is the $\ell$-dimensional canonical vector which is a zero vector except that the $i$-th entry is $1$, i.e. 
\begin{align*}
    \beta_i=
    \begin{bmatrix}
    \beta_{i,0}&
    \beta_{i,1} &
    \cdots &
    \beta_{i,\ell-1}
    \end{bmatrix}^\top
\end{align*}
and
\begin{align*}
    \mathbf{e}_i=
    \begin{bmatrix}
    0 &
    \cdots &
    1 &
    \cdots &
    0
    \end{bmatrix}^\top.
\end{align*}
Namely, we have $\beta_i=V^{-1}\mathbf{e}_i$.
Recall that the definition of $\mathcal{E}_i$ is 
\begin{align*}
    \mathcal{E}_i=\sum_{j=\ell}^\infty J_{f_i,j}\cdot \mathbf{m}_j.
\end{align*}
To bound the term $J_{f_i,j}$, observe that 
\begin{align*}
    J_{f_i,j} = \sum_{a=0}^{\ell-1}\beta_{i,a}J_{h_a,j} = (v^{(j)})^\top V^{-1}\mathbf{e}_i
\end{align*}
and, by Cramer's rule, $J_{f_i,j}$ can be expressed as
\begin{align*}
    J_{f_i,j}=\frac{\det(V^{(i\rightarrow j)})}{\det(V)}
\end{align*}
where $V^{(i\rightarrow j)}$ is the same matrix as $V$ except that the $i$-th row is replaced with $v^{(j)}$, i.e.
\begin{align*}
    V^{(i\rightarrow j)} 
    &=
    \begin{bmatrix}
    v^{(0)} &  \cdots & v^{(i-1)} & v^{(j)} & v^{(i+1)} & \cdots & v^{(\ell-1)}
    \end{bmatrix}^\top \numberthis\label{eq:v_i_j_def}
\end{align*}
for $i=1,\dots,2k-1$ and $j\geq \ell$.
The right arrow in the superscript indicates the row replacement.
We preview that there are column replacements in our calculation and we will use left arrows to indicate it.

In Lemma \ref{lem:det_v_i_j}, we show that 
\begin{align*}
    \abs{J_{f_i,j}}=\frac{\abs{\det(V^{(i\rightarrow j)})}}{\abs{\det(V)}}
    \leq
    \frac{1}{2^{\Omega(j\log j)}}.
\end{align*}
Also, by the assumption that $\abs{\mu_i}<M$ where $M$ is a constant, we have $\mathbf{m}_j \leq M^j$.
Hence, we prove that
\begin{align*}
    \abs{\mathcal{E}_i}
    & \leq
    \sum_{j=\ell}^\infty \abs{J_{f_i,j}}\abs{\mathbf{m}_j}
    \leq
    \sum_{j=\ell}^\infty\frac{1}{2^{\Omega(j\log j)}} \cdot M^j 
    \leq
    \frac{1}{2^{\Omega(\ell\log \ell)}} \cdot M^\ell
    \leq \delta
\end{align*}
as long as $\ell=\Omega(\frac{\log \frac{1}{\delta}}{\log\log \frac{1}{\delta}})$.

Hence, we have the following lemma.
\begin{lemma}\label{lem:ell}
For a sufficiently small $\delta>0$, when $\ell=\Omega(\frac{\log\frac{1}{\delta}}{\log\log\frac{1}{\delta}})$, the estimators $\widehat{\mathbf{m}}_i$ computed by \eqref{eq:estimator} satisfies
\begin{align*}
    \abs{\widehat{\mathbf{m}}_i - \mathbf{m}_i} < 2\delta \qquad \text{with probability $1-\frac{\Var(\widehat{\mathbf{m}}_i)}{\delta^2}$.}
\end{align*}
\end{lemma}

\subsection{Bounds for the Variance}\label{sec:how_large_var}

Recall that our second subproblem is to bound the variance of our estimator.
To bound $\Var(\widehat{\mathbf{m}}_i)$, observe that
\begin{align*}
    \Var(\widehat{\mathbf{m}}_i)
    & \leq
    \E(\widehat{\mathbf{m}}_i^2) 
    = 
    \frac{\alpha}{n}\E(f_i(X)^2) \\
    & =
    \frac{\alpha}{n}\E\bigg(\big(\sum_{a=0}^{\ell-1}\beta_{i,a}h_a(X)\big)^2\bigg) \\
    & \leq
    \frac{\alpha}{n}\left(\sum_{a=0}^{\ell-1}\abs{\beta_{i,a}}\sqrt{\E\left(h_a(X)^2\right)}\right)^2 \numberthis\label{eq:main_var}
\end{align*}
By expanding the expectation explicitly,
\begin{align*}
    \E\left(h_a(X)^2\right)
    & =
    \int_{S} h_a(x)^2\cdot \left(\sum_{i=1}^k w_i g_{\mu_i,1,S}(x)\right) \dir x \\
    & \leq
    \frac{1}{\alpha} \int_{\mathbb{R}} h_a(x)^2\cdot \left(\sum_{i=1}^k w_i g_{\mu_i,1}(x)\right) \dir x \\
    & \leq
    \frac{1}{\alpha}(O(M+\sqrt{a}))^{2a} \numberthis\label{eq:exp_hermite_2}
\end{align*}
The last line comes from \cite{wu2018optimal} where they showed that
\begin{align*}
    \int_{\mathbb{R}} h_a(x)^2\cdot \left(\sum_{i=1}^k w_i g_{\mu_i,1}(x)\right) \dir x \leq (O(M+\sqrt{a}))^{2a}
\end{align*}
in Lemma 5 of \cite{wu2018optimal}.

Now, we also need to bound $\abs{\beta_{i,a}}$.
Recall that 
\begin{align*}
    \beta_i =V^{-1}\mathbf{e}_i.
\end{align*}
By Cramer's rule,  each coordinate of $\beta_i$ is 
\begin{align*}
    \beta_{i,a}=\frac{\det(V^{(\mathbf{e}_i\leftarrow a)})}{\det(V)}
\end{align*}
where $V^{(\mathbf{e}_i\leftarrow a)}$ is the same matrix as $V$ except that the $a$-th column is replaced with $\mathbf{e}_i$, i.e.
\begin{align*}
    & V^{(\mathbf{e}_i\leftarrow a)} \\
    & =
    \begin{bmatrix}
    v^{(0)}_0 & \cdots & v^{(0)}_{a-1} & 0 & v^{(0)}_{a+1} & \cdots & v^{(0)}_{\ell-1} \\
    \vdots & \ddots & \vdots & \vdots & \vdots & \ddots & \vdots \\
    v^{(i)}_0 & \cdots & v^{(i)}_{a-1} & 1 & v^{(i)}_{a+1} & \cdots & v^{(i)}_{\ell-1} \\
    \vdots & \ddots & \vdots & \vdots & \vdots & \ddots & \vdots \\
    v^{(\ell-1)}_0 & \cdots & v^{(\ell-1)}_{a-1} & 0 & v^{(\ell-1)}_{a+1} & \cdots & v^{(\ell-1)}_{\ell-1} 
    \end{bmatrix}\numberthis\label{eq:v_i_a_def}
\end{align*}

In Lemma \ref{lem:det_v_i_a}, we show that 
\begin{align*}
    \abs{\beta_{i,a}}
    & \leq 
    2^{O(\ell\log \ell )}. \numberthis\label{eq:beta_i_a}
\end{align*}
Therefore, if we plug \eqref{eq:exp_hermite_2} and \eqref{eq:beta_i_a} into \eqref{eq:main_var}, we have the following lemma.
\begin{lemma}\label{lem:var}
For any positive integer $\ell$, the estimator $\widehat{\mathbf{m}}_i$ computed by \eqref{eq:estimator} has variance
\begin{align*}
    \Var(\widehat{\mathbf{m}}_i)
    & \leq
    \frac{1}{n}\cdot 2^{O(\ell\log\ell)}.
\end{align*}
\end{lemma}

\subsection{Full Algorithm and Main Theorem}

\begin{algorithm}[!t]
	\caption{Learning mixtures of Gaussians with censored data}\label{alg:main}
	{\bf Input:} $n$ iid samples $x_1,\dots,x_n$, number of Gaussians $k$, parameter $\ell$, mean boundary parameter $M$, sample domain $S=[-R,R]$
	\begin{algorithmic}[1]
	    \FOR{$i=0$ to $2k-1$}
	    \STATE solve \eqref{eq:sys_eq} to obtain $\beta_{i} = (\beta_{i,0}, \beta_{i,1},\dots,\beta_{i,\ell-1})^\top$, i.e. solve the following system of linear equations
	    \begin{align*}
	        V\beta_i=\mathbf{e}_i
	    \end{align*}
	    where the $(r,c)$-entry of $V$ is 
	    \begin{align*}
	        \int_{S}  \frac{1}{\sqrt{2\pi}r!}e^{-\frac{1}{2}x^2}h_c(x)h_r(x) \dir x \numberthis \label{eq:alg_int}
	    \end{align*}
	    and $\mathbf{e}_i$ is the canonical vector
	    \ENDFOR
		\FOR{each sample $x_s$}
		\STATE compute $\widehat f_i(x_s) := \begin{cases}
		    f_i(x_s) &\text{if $x_s$ is non-\textsf{FAIL}} \\
		    0 &\text{if $x_s$ is \textsf{FAIL}}
		\end{cases}$; \mbox{recall} that $f_i$ is
		\begin{align*}
		    f_i(x) = \sum_{a=0}^{\ell-1} \beta_{i,a}h_a(x)
		\end{align*}
		and $h_a$ is the $a$-th Hermite polynomial
		\begin{align*}
		    h_a(x) = a!\sum_{j=0}^{\lfloor a/2\rfloor} \frac{(-1/2)^j}{j!(a-2j)!} x^{a-2j}
		\end{align*}
        \ENDFOR
        \FOR{$i=1$ to $2k-1$}
        \STATE compute $\widehat{\mathbf{m}}_i = \frac{1}{n}\sum_{s=1}^n \widehat f_i(x_s)$ which is the same as the estimator defined in $\eqref{eq:estimator}$
        \ENDFOR
        \STATE let $\widehat{w}_1,\widehat{w}_2,\dots,\widehat{w}_k$ and  $\widehat{\mu}_1,\widehat{\mu}_2,\cdots,\widehat{\mu}_k$ be the output of Algorithm \ref{alg:denoise} using $\widehat{\mathbf{m}} = (\widehat{\mathbf{m}}_1,\dots,\widehat{\mathbf{m}}_{2k-1})$ and $M$ as the input
	\end{algorithmic}
	{\bf Output:} estimated weights $\widehat{w}_1,\widehat{w}_2,\dots,\widehat{w}_k$ and estimated means $\widehat{\mu}_1,\widehat{\mu}_2,\cdots,\widehat{\mu}_k$
\end{algorithm}

In this subsection, we will present the full algorithm and combine with the analysis in the previous subsections to prove our main theorem.

\begin{proof}[Proof of Theorem \ref{thm:main}]

Suppose we are given $n$ iid samples $x_1,\dots,x_n$ from the distribution \eqref{eq:model}.
We will show that the estimated weights $\widehat{w}_1,\widehat{w}_2,\dots,\widehat{w}_k$ and the estimated means $\widehat{\mu}_1,\widehat{\mu}_2,\cdots,\widehat{\mu}_k$ outputted by Algorithm \ref{alg:main} taking $x_1,\dots,x_n$ as the input satisfy the desired guarantees.

By Lemma \ref{lem:ell}, when $\ell = \Omega(\frac{\log \frac{1}{\delta}}{\log\log\frac{1}{\delta}})$, we have
\begin{align*}
    \abs{\widehat{\mathbf{m}}_i - \mathbf{m}_i} < 2\delta \qquad \text{with probability $1-\frac{\Var(\widehat{\mathbf{m}}_i)}{\delta^2}$}
\end{align*}
where $\widehat{\mathbf{m}}_i$ are computed in Algorithm \ref{alg:main}.
Moreover, by Lemma \ref{lem:var}, we show that 
\begin{align*}
    \Var(\widehat{\mathbf{m}}_i)
    & \leq
    \frac{1}{n}\cdot 2^{O(\ell\log \ell )}
\end{align*}
which implies when $\ell = \Omega(\frac{\log \frac{1}{\delta}}{\log\log\frac{1}{\delta}})$ the failure probability is less than 
\begin{align*}
    \frac{\Var(\widehat{\mathbf{m}}_i)}{\delta^2}
    & \leq
    \frac{1}{n}\cdot \poly(\frac{1}{\delta}).
\end{align*}
By applying the union bound over all $i=1,2,\dots,2k-1$, when $n = \Omega( \poly(\frac{1}{\delta}))$, we have
\begin{align*}
    \abs{\widehat{\mathbf{m}}_i - \mathbf{m}_i} < 2\delta \qquad \text{with probability $\frac{99}{100}$.}
\end{align*}

In \cite{wu2018optimal}, they showed that Algorithm \ref{alg:denoise} is the algorithm that makes the guarantees hold in Theorem \ref{thm:denoised}.
Therefore, if we pick $\delta=\eps^{\Omega(k)}$ along with the assumption $w_{\min}\Delta_{\min} = \Omega(\eps)$, we have, up to an index permutation $\Pi$,
\begin{align*}
    \abs{\widehat{w}_{\Pi(i)}-w_i}<\eps,\qquad \abs{\widehat{\mu}_{\Pi(i)}-\mu_i}<\eps\qquad \text{for $i=1,\dots,k$}.
\end{align*}

We now examine the running time of Algorithm \ref{alg:main}. 
It first takes $k\cdot \poly(\ell)$ time\footnote{Computing the integral in \eqref{eq:alg_int} can be reduced to computing the integral $\int_0^z e^{-\frac{1}{2}t^2}\dir t$ by observing $h_c$ and $h_r$ are polynomials and using integration by parts.
If we remove the assumption that the exact computation can be done, we will need to approximate the integral up to an additive error of $1/2^{\poly(k,\ell)}$.
One can approximate the integral in an exponential convergence rate by Taylor expansion and hence the running time is still $k\cdot \poly(\ell)$ for this step.} 
to obtain $\beta_i$.
Then, it takes $n\cdot k\cdot \poly(\ell)$ to compute $\widehat{\mathbf{m}}_i$.
Finally, the running time for Algorithm \ref{alg:denoise} is $\poly(k)$.
Hence, by plugging $\ell = O(k\log\frac{1}{\eps})$, the running time of Algorithm \ref{alg:main} is $n\cdot\poly(k,\frac{1}{\eps})$.

\end{proof}

\begin{algorithm}[t]
	\caption{Denoised method of moments \cite{wu2018optimal}}\label{alg:denoise}
	{\bf Input:} estimated moments $\widehat{\mathbf{m}} = (\widehat{\mathbf{m}}_1,\dots,\widehat{\mathbf{m}}_{2k-1})$, mean boundary parameter $M$
	\begin{algorithmic}[1]
	    \STATE let $\mathbf{m}^* = (\mathbf{m}^*_1,\dots,\mathbf{m}^*_{2k-1})$ be the optimal solution of the following convex optimization problem 
	    \begin{align*}
	        & \arg\max_{\mathbf{m}} \norm{\widehat{\mathbf{m}} - \mathbf{m}} \\
	        \text{s.t. } & M\cdot \mathbf{M}_{0,2k-2} \succcurlyeq \mathbf{M}_{1,2k-1} \succcurlyeq -M \cdot\mathbf{M}_{0,2k-2}
	    \end{align*}
	    where $\mathbf{M}_{i,j}$ is the Hankel matrix whose entries are $\mathbf{m}_i,\dots,\mathbf{m}_{j}$, i.e.
	    \begin{align*}
	        \mathbf{M}_{i,j}=
	        \begin{bmatrix}
	        \mathbf{m}_i & \mathbf{m}_{i+1}& \cdots & \mathbf{m}_{\frac{i+j}{2}}\\
	        \mathbf{m}_{i+1} & \mathbf{m}_{i+2}& \cdots & \mathbf{m}_i \\
	        \vdots & \vdots& \ddots & \vdots\\
	        \mathbf{m}_{\frac{i+j}{2}} & \mathbf{m}_{\frac{i+j}{2}+1}& \cdots & \mathbf{m}_j 
	        \end{bmatrix}
	    \end{align*}
	    \STATE let $\widehat{\mu}_1,\widehat{\mu}_2,\cdots,\widehat{\mu}_k$ be the roots of the polynomial $P$ where 
	    \begin{align*}
	        P(x) = \det\begin{bmatrix}
	        1 & \mathbf{m}^*_1 & \cdots & \mathbf{m}^*_{k} \\
	        \vdots & \vdots & \ddots & \vdots \\
	        \mathbf{m}^*_{k-1} & \mathbf{m}^*_k & \cdots & \mathbf{m}^*_{2k-1} \\
	        1 & x & \cdots & x^k
	        \end{bmatrix}
	    \end{align*}
	    \STATE let $(\widehat{w}_1,\widehat{w}_2,\dots,\widehat{w}_k)^\top$ be the solution of the following system of linear equations
	    \begin{align*}
	        \begin{bmatrix}
	         1 & 1 & \cdots & 1 \\
	         \widehat{\mu}_1 & \widehat{\mu}_2 & \cdots & \widehat{\mu}_k \\
	         \vdots & \vdots & \ddots & \vdots \\
	         \widehat{\mu}_1^{k-1} & \widehat{\mu}_2^{k-1} & \cdots & \widehat{\mu}_k^{k-1}
	        \end{bmatrix}
	        \begin{bmatrix}
	        w_1 \\
	        w_2 \\
	        \vdots \\
	        w_k
	        \end{bmatrix}
	        =
	        \begin{bmatrix}
	        1 \\
	        \mathbf{m}^*_1 \\
	        \vdots \\
	        \mathbf{m}^*_{k-1}
	        \end{bmatrix}
	    \end{align*}
	\end{algorithmic}
	{\bf Output:} estimated weights $\widehat{w}_1,\widehat{w}_2,\dots,\widehat{w}_k$ and estimated means $\widehat{\mu}_1,\widehat{\mu}_2,\cdots,\widehat{\mu}_k$
\end{algorithm}

\section{Conclusion and Discussion}

In this paper, we study the classical problem of learning mixtures of Gaussians with censored data.
The problem becomes more challenging compared to the problem of learning with uncensored data because the data are partially observed.
Our result shows that there is an efficient algorithm to estimate the weights and the means of the Gaussians.
Specifically, we show that one only needs $\frac{1}{\eps^{O(k)}}$ censored samples to estimate the weights and the means within $\eps$ error.
To the best of our knowledge, this is the first finite sample bound for the problem of learning mixtures of Gaussians with censored data even in the simple setting that the Gaussians are univariate and homogeneous.

There are multiple natural extensions to this setting.
For example, a natural extension is to consider mixtures of multivariate Gaussians.
% One can extend our setup to multivariate Gaussians.
Without truncation or censoring, one popular approach to learn mixtures of multivariate Gaussians is to apply random projections and reduce the problem to univariate Gaussians.
This approach relies on the fact that the projection of a mixture of Gaussians is also a mixture of Gaussians.
Unfortunately, this fact is no longer true when the data are truncated or censored.

Another interesting direction is to relax the assumption of known and homogeneous variances to unknown and/or non-homogeneous variances.
When the Gaussians are homogeneous, one can estimate the variance by computing the pairwise distances between $k+1$ samples and find the minimum of them if the samples are not truncated or censored.
It holds from the fact that two samples are from the same Gaussian and hence the expected value of their squared distance is the variance.
It becomes more challenging when the samples are truncated or censored because the expected value of the squared distance may not be the variance.

Furthermore, previous results indicate that, in the uncensored setting, sample bounds can be improved when the centers of Gaussians in the mixture are well-separated \cite{moitra2015super,regev2017learning,qiao2022fourier}. 
An interesting direction for future research would be to improve our results under stronger separation assumptions on the components.
For example, one strategy to exploit separation is to apply the Fourier Transform to the pdf of the mixture. 
% Specifically, we need to integrate over the entire space or require uncensored samples to estimate the Fourier Transform. 
With uncensored samples, it is straightforward to estimate the Fourier Transform, however, when the pdf is truncated, a challenge arises as the Fourier Transform may not yield a convenient form, as required by these analyses. 
We anticipate that delicate modifications may still be needed, and leave this open to future work.

\bibliographystyle{abbrvnat}
\bibliography{ref}

\begin{thebibliography}{39}
\providecommand{\natexlab}[1]{#1}
\providecommand{\url}[1]{\texttt{#1}}
\expandafter\ifx\csname urlstyle\endcsname\relax
  \providecommand{\doi}[1]{doi: #1}\else
  \providecommand{\doi}{doi: \begingroup \urlstyle{rm}\Url}\fi

\bibitem[Amemiya(1973)]{amemiya1973regression}
T.~Amemiya.
\newblock Regression analysis when the dependent variable is truncated normal.
\newblock \emph{Econometrica: Journal of the Econometric Society}, pages
  997--1016, 1973.

\bibitem[Balakrishnan and Cramer(2014)]{balakrishnan2014art}
N.~Balakrishnan and E.~Cramer.
\newblock The art of progressive censoring.
\newblock \emph{Statistics for industry and technology}, 2014.

\bibitem[Bernoulli(1760)]{bernoulli1760essai}
D.~Bernoulli.
\newblock Essai d'une nouvelle analyse de la mortalit{\'e} caus{\'e}e par la
  petite v{\'e}role, et des avantages de l'inoculation pour la pr{\'e}venir.
\newblock \emph{Histoire de l'Acad., Roy. Sci.(Paris) avec Mem}, pages 1--45,
  1760.

\bibitem[Cohen(2016)]{cohen2016truncated}
A.~C. Cohen.
\newblock \emph{Truncated and Censored Samples: Theory and Applications}.
\newblock CRC Press, 2016.

\bibitem[Dasgupta(1999)]{dasgupta1999learning}
S.~Dasgupta.
\newblock Learning mixtures of gaussians.
\newblock In \emph{40th Annual Symposium on Foundations of Computer Science
  (Cat. No. 99CB37039)}, pages 634--644. IEEE, 1999.

\bibitem[Daskalakis et~al.(2017)Daskalakis, Tzamos, and
  Zampetakis]{daskalakis2017ten}
C.~Daskalakis, C.~Tzamos, and M.~Zampetakis.
\newblock Ten steps of em suffice for mixtures of two gaussians.
\newblock In \emph{Conference on Learning Theory}, pages 704--710. PMLR, 2017.

\bibitem[Daskalakis et~al.(2018)Daskalakis, Gouleakis, Tzamos, and
  Zampetakis]{daskalakis2018efficient}
C.~Daskalakis, T.~Gouleakis, C.~Tzamos, and M.~Zampetakis.
\newblock Efficient statistics, in high dimensions, from truncated samples.
\newblock In \emph{2018 IEEE 59th Annual Symposium on Foundations of Computer
  Science (FOCS)}, pages 639--649. IEEE, 2018.

\bibitem[Daskalakis et~al.(2019)Daskalakis, Gouleakis, Tzamos, and
  Zampetakis]{daskalakis2019computationally}
C.~Daskalakis, T.~Gouleakis, C.~Tzamos, and M.~Zampetakis.
\newblock Computationally and statistically efficient truncated regression.
\newblock In \emph{Conference on Learning Theory}, pages 955--960. PMLR, 2019.

\bibitem[Diakonikolas and Kane(2019)]{diakonikolas2019recent}
I.~Diakonikolas and D.~M. Kane.
\newblock Recent advances in algorithmic high-dimensional robust statistics.
\newblock \emph{arXiv preprint arXiv:1911.05911}, 2019.

\bibitem[Diakonikolas et~al.(2017)Diakonikolas, Kamath, Kane, Li, Moitra, and
  Stewart]{diakonikolas2017being}
I.~Diakonikolas, G.~Kamath, D.~M. Kane, J.~Li, A.~Moitra, and A.~Stewart.
\newblock Being robust (in high dimensions) can be practical.
\newblock In \emph{International Conference on Machine Learning}, pages
  999--1008. PMLR, 2017.

\bibitem[Diakonikolas et~al.(2018)Diakonikolas, Kamath, Kane, Li, Moitra, and
  Stewart]{diakonikolas2018robustly}
I.~Diakonikolas, G.~Kamath, D.~M. Kane, J.~Li, A.~Moitra, and A.~Stewart.
\newblock Robustly learning a gaussian: Getting optimal error, efficiently.
\newblock In \emph{Proceedings of the Twenty-Ninth Annual ACM-SIAM Symposium on
  Discrete Algorithms}, pages 2683--2702. SIAM, 2018.

\bibitem[Diakonikolas et~al.(2019)Diakonikolas, Kamath, Kane, Li, Moitra, and
  Stewart]{diakonikolas2019robust}
I.~Diakonikolas, G.~Kamath, D.~Kane, J.~Li, A.~Moitra, and A.~Stewart.
\newblock Robust estimators in high-dimensions without the computational
  intractability.
\newblock \emph{SIAM Journal on Computing}, 48\penalty0 (2):\penalty0 742--864,
  2019.

\bibitem[Doss et~al.(2020)Doss, Wu, Yang, and Zhou]{doss2020optimal}
N.~Doss, Y.~Wu, P.~Yang, and H.~H. Zhou.
\newblock Optimal estimation of high-dimensional gaussian mixtures.
\newblock \emph{arXiv preprint arXiv:2002.05818}, 2020.

\bibitem[Fisher(1931)]{fisher1931properties}
R.~Fisher.
\newblock Properties and applications of hh functions.
\newblock \emph{Mathematical tables}, 1:\penalty0 815--852, 1931.

\bibitem[Galton(1898)]{galton1898examination}
F.~Galton.
\newblock An examination into the registered speeds of american trotting
  horses, with remarks on their value as hereditary data.
\newblock \emph{Proceedings of the Royal Society of London}, 62\penalty0
  (379-387):\penalty0 310--315, 1898.

\bibitem[Hardt and Price(2015)]{hardt2015tight}
M.~Hardt and E.~Price.
\newblock Tight bounds for learning a mixture of two gaussians.
\newblock In \emph{Proceedings of the forty-seventh annual ACM symposium on
  Theory of computing}, pages 753--760, 2015.

\bibitem[Hausman and Wise(1977)]{hausman1977social}
J.~A. Hausman and D.~A. Wise.
\newblock Social experimentation, truncated distributions, and efficient
  estimation.
\newblock \emph{Econometrica: Journal of the Econometric Society}, pages
  919--938, 1977.

\bibitem[Hopkins et~al.(2022)Hopkins, Kamath, and Majid]{hopkins2022efficient}
S.~B. Hopkins, G.~Kamath, and M.~Majid.
\newblock Efficient mean estimation with pure differential privacy via a
  sum-of-squares exponential mechanism.
\newblock In \emph{Proceedings of the 54th Annual ACM SIGACT Symposium on
  Theory of Computing}, pages 1406--1417, 2022.

\bibitem[Kalai et~al.(2010)Kalai, Moitra, and Valiant]{kalai2010efficiently}
A.~T. Kalai, A.~Moitra, and G.~Valiant.
\newblock Efficiently learning mixtures of two gaussians.
\newblock In \emph{Proceedings of the forty-second ACM symposium on Theory of
  computing}, pages 553--562, 2010.

\bibitem[Lai et~al.(2016)Lai, Rao, and Vempala]{lai2016agnostic}
K.~A. Lai, A.~B. Rao, and S.~Vempala.
\newblock Agnostic estimation of mean and covariance.
\newblock In \emph{2016 IEEE 57th Annual Symposium on Foundations of Computer
  Science (FOCS)}, pages 665--674. IEEE, 2016.

\bibitem[Lee(1914)]{lee1914table}
A.~Lee.
\newblock Table of the gaussian" tail" functions; when the" tail" is larger
  than the body.
\newblock \emph{Biometrika}, 10\penalty0 (2/3):\penalty0 208--214, 1914.

\bibitem[Lee and Scott(2012)]{lee2012algorithms}
G.~Lee and C.~Scott.
\newblock Em algorithms for multivariate gaussian mixture models with truncated
  and censored data.
\newblock \emph{Computational Statistics \& Data Analysis}, 56\penalty0
  (9):\penalty0 2816--2829, 2012.

\bibitem[Lindsay(1995)]{lindsay1995}
B.~G. Lindsay.
\newblock Mixture models: theory, geometry and applications.
\newblock In \emph{NSF-CBMS regional conference series in probability and
  statistics}, pages i--163. JSTOR, 1995.

\bibitem[Liu et~al.(2021)Liu, Kong, Kakade, and Oh]{liu2021robust}
X.~Liu, W.~Kong, S.~Kakade, and S.~Oh.
\newblock Robust and differentially private mean estimation.
\newblock \emph{Advances in neural information processing systems},
  34:\penalty0 3887--3901, 2021.

\bibitem[Maddala(1986)]{maddala1986limited}
G.~S. Maddala.
\newblock \emph{Limited-dependent and qualitative variables in econometrics}.
\newblock Number~3. Cambridge university press, 1986.

\bibitem[McLachlan and Jones(1988)]{mclachlan1988fitting}
G.~McLachlan and P.~Jones.
\newblock Fitting mixture models to grouped and truncated data via the em
  algorithm.
\newblock \emph{Biometrics}, pages 571--578, 1988.

\bibitem[Moitra(2015)]{moitra2015super}
A.~Moitra.
\newblock Super-resolution, extremal functions and the condition number of
  vandermonde matrices.
\newblock In \emph{Proceedings of the forty-seventh annual ACM symposium on
  Theory of computing}, pages 821--830, 2015.

\bibitem[Moitra and Valiant(2010)]{moitra2010settling}
A.~Moitra and G.~Valiant.
\newblock Settling the polynomial learnability of mixtures of gaussians.
\newblock In \emph{2010 IEEE 51st Annual Symposium on Foundations of Computer
  Science}, pages 93--102. IEEE, 2010.

\bibitem[Nagarajan and Panageas(2020)]{nagarajan2020analysis}
S.~G. Nagarajan and I.~Panageas.
\newblock On the analysis of em for truncated mixtures of two gaussians.
\newblock In \emph{Algorithmic Learning Theory}, pages 634--659. PMLR, 2020.

\bibitem[Pearson(1894)]{pearson1894contributions}
K.~Pearson.
\newblock Contributions to the mathematical theory of evolution.
\newblock \emph{Philosophical Transactions of the Royal Society of London. A},
  185:\penalty0 71--110, 1894.

\bibitem[Pearson(1902)]{pearson1902systematic}
K.~Pearson.
\newblock On the systematic fitting of curves to observations and measurements.
\newblock \emph{Biometrika}, 1\penalty0 (3):\penalty0 265--303, 1902.

\bibitem[Pearson and Lee(1908)]{pearson1908generalised}
K.~Pearson and A.~Lee.
\newblock On the generalised probable error in multiple normal correlation.
\newblock \emph{Biometrika}, 6\penalty0 (1):\penalty0 59--68, 1908.

\bibitem[Qiao et~al.(2022)Qiao, Guruganesh, Rawat, Dubey, and
  Zaheer]{qiao2022fourier}
M.~Qiao, G.~Guruganesh, A.~Rawat, K.~A. Dubey, and M.~Zaheer.
\newblock A fourier approach to mixture learning.
\newblock \emph{Advances in Neural Information Processing Systems},
  35:\penalty0 20850--20861, 2022.

\bibitem[Regev and Vijayaraghavan(2017)]{regev2017learning}
O.~Regev and A.~Vijayaraghavan.
\newblock On learning mixtures of well-separated gaussians.
\newblock In \emph{2017 IEEE 58th Annual Symposium on Foundations of Computer
  Science (FOCS)}, pages 85--96. IEEE, 2017.

\bibitem[Schneider(1986)]{schneider1986truncated}
H.~Schneider.
\newblock \emph{Truncated and censored samples from normal populations}.
\newblock Marcel Dekker, Inc., 1986.

\bibitem[Tobin(1958)]{tobin1958estimation}
J.~Tobin.
\newblock Estimation of relationships for limited dependent variables.
\newblock \emph{Econometrica: journal of the Econometric Society}, pages
  24--36, 1958.

\bibitem[Vempala and Wang(2004)]{vempala2004spectral}
S.~Vempala and G.~Wang.
\newblock A spectral algorithm for learning mixture models.
\newblock \emph{Journal of Computer and System Sciences}, 68\penalty0
  (4):\penalty0 841--860, 2004.

\bibitem[Wu and Yang(2018)]{wu2018optimal}
Y.~Wu and P.~Yang.
\newblock Optimal estimation of gaussian mixtures via denoised method of
  moments.
\newblock \emph{arXiv preprint arXiv:1807.07237}, 2018.

\bibitem[Xu et~al.(2016)Xu, Hsu, and Maleki]{xu2016global}
J.~Xu, D.~J. Hsu, and A.~Maleki.
\newblock Global analysis of expectation maximization for mixtures of two
  gaussians.
\newblock \emph{Advances in Neural Information Processing Systems}, 29, 2016.

\end{thebibliography}

\appendix

\section{Proof}

In this section, we will present the proofs of the lemmas.

\begin{lemma}\label{lem:det_v}

Let $V$ be the matrix defined in \eqref{eq:v_def}, i.e. $V$ is the $\ell$-by-$\ell$ matrix whose $(r,c)$-entry is $J_{h_{c},r}$ for $r,c=0,1,\dots,\ell-1$.
Recall that, from \eqref{eq:j_def}, $J_{h_{c},r}$ is defined as
\begin{align*}
    J_{h_c,r} = \int_{S} \frac{1}{\sqrt{2\pi}r!}e^{-\frac{1}{2}x^2}h_c(x)h_r(x) \dir x.
\end{align*}
Then, the determinant of $V$ is 
\begin{align*}
    \det(V)
    =
    \left(\frac{1}{\sqrt{2\pi}}\right)^\ell\cdot \prod_{r=0}^{\ell-1}\frac{1}{r!} \cdot\int_{x_0>\cdots > x_{\ell-1},\mathbf{x}\in S^{\ell}} e^{-\frac{1}{2}\sum_{c=0}^{\ell-1}x_c^2}\cdot\prod_{0\leq c_1<c_2 \leq \ell-1} (x_{c_1}-x_{c_2})^2 \dir \mathbf{x}.
\end{align*}

\end{lemma}

\begin{proof}

Since the $(r,c)$-entry of $V$ is 
\begin{align*}
    J_{h_c,r}=\int_{S}  \frac{1}{\sqrt{2\pi}r!}e^{-\frac{1}{2}x^2}h_c(x)h_r(x) \dir x,
\end{align*}
by factoring out the term $\frac{1}{\sqrt{2\pi}r!}$ for each row, we have 
\begin{align*}
    \det(V)
    =
    \left(\frac{1}{\sqrt{2\pi}}\right)^\ell \cdot \prod_{r=0}^{\ell-1}\frac{1}{r!}\cdot \det(W) \numberthis\label{eq:v_w_connection}
\end{align*}
where $W$ is the $\ell$-by-$\ell$ matrix whose $(r,c)$-entry is 
\begin{align*}
    W_{r,c}=\int_{S} e^{-\frac{1}{2}x^2}h_c(x)h_r(x) \dir x.\numberthis\label{eq:w_def}
\end{align*}
By Cauchy-Binet formula, we can further express $\det(W)$ as 
\begin{align*}
    \det(W)
    =
    \int_{x_0>\cdots > x_{\ell-1},\mathbf{x}\in S^{\ell}} (\det(U(\mathbf{x})))^2 \dir \mathbf{x} \numberthis\label{eq:w_u_connection}
\end{align*}
where $U(\mathbf{x})$ is the $\ell$-by-$\ell$ matrix whose $(r,c)$-entry is 
\begin{align*}
    U(\mathbf{x})_{r,c} = e^{-\frac{1}{4}x_c^2}h_r(x_c)\numberthis\label{eq:u_def}
\end{align*}
for any $\mathbf{x}=(x_0,\dots, x_{\ell-1})\in S^\ell$.
By factoring out the term $e^{-\frac{1}{4}x_c^2}$ for each column, we have
\begin{align*}
    \det(U(\mathbf{x})) 
    =
    e^{-\frac{1}{4}\sum_{c=0}^{\ell-1}x_c^2}\det(P(\mathbf{x})) \numberthis\label{eq:u_p_connection}
\end{align*}
where $P(\mathbf{x})$ is the $\ell$-by-$\ell$ matrix whose $(r,c)$-entry is 
\begin{align*}
    P(\mathbf{x})_{r,c} = h_r(x_c) \numberthis\label{eq:p_def}
\end{align*}
for any $\mathbf{x}=(x_0,\dots, x_{\ell-1})\in S^\ell$.
Since $h_r$ is a polynomial of degree $r$ with the leading coefficient $1$, by applying row and column operations, the determinant $\det(P(\mathbf{x}))$ is same as the determinant of the Vandermonde matrix, i.e.
\begin{align*}
    \det(P(\mathbf{x})) = \prod_{0\leq c_1<c_2 \leq \ell-1} (x_{c_1}-x_{c_2}). \numberthis\label{eq:det_p}
\end{align*}
In other words, the determinant $\det(V)$ is 
\begin{align*}
    \det(V)
    =
    \left(\frac{1}{\sqrt{2\pi}}\right)^\ell\cdot \prod_{r=0}^{\ell-1}\frac{1}{r!} \cdot\int_{x_0>\cdots > x_{\ell-1},\mathbf{x}\in S^{\ell}} e^{-\frac{1}{2}\sum_{c=0}^{\ell-1}x_c^2}\cdot\prod_{0\leq c_1<c_2 \leq \ell-1} (x_{c_1}-x_{c_2})^2 \dir \mathbf{x}.
\end{align*}

\end{proof}

\begin{lemma}\label{lem:det_v_i_j}

Let $V^{(i\rightarrow j)}$ be the matrix defined in \eqref{eq:v_i_j_def} for $i\leq 2k-1$ and $j\geq \ell \geq 2(2k-1) \geq 2i$.
Then the absolute value of the  determinant of $V^{(i\rightarrow j)}$ is 
\begin{align*}
    \abs{\det(V^{(i\rightarrow j)})} \leq \frac{1}{2^{\Omega(j\log j)}} \cdot \abs{\det(V)}.
\end{align*}

\end{lemma}

\begin{proof}

We can perform a similar computation as in the computation of $\det(V)$.
Namely, we factor out the term $\frac{1}{\sqrt{2\pi}r!}$ for each row, we have
\begin{align*}
    \abs{\det(V^{(i\rightarrow j)})}
    =
    \left(\frac{1}{\sqrt{2\pi}}\right)^\ell \cdot \prod_{r=0,r\neq i}^{\ell-1}\frac{1}{r!}\cdot \frac{1}{j!}\cdot\abs{\det(W^{(i\rightarrow j)})}
\end{align*}
where $W^{(i\rightarrow j)}$ is the same matrix as $W$ from \eqref{eq:w_def} except that the $i$-th row is replaced by the row $\sqrt{2\pi}j!v^{(j)}$.
By comparing to \eqref{eq:v_w_connection}, we simplify $\abs{\det(V^{(i\rightarrow j)})}$ to be
\begin{align*}
    \abs{\det(V^{(i\rightarrow j)})}
    =
    \frac{i!}{j!}\cdot \frac{\abs{\det(W^{(i\rightarrow j)})}}{\abs{\det(W)}}\cdot \abs{\det(V)} \numberthis\label{eq:det_v_eq}
\end{align*}

By Cauchy-Binet formula, we can further express $\det(W^{(i\rightarrow j)})$ as
\begin{align*}
    \det(W^{(i\rightarrow j)})
    =
    \int_{x_0>\cdots>x_{\ell-1},\mathbf{x}\in S^{\ell}} \det(U(\mathbf{x}))\det(U^{(i\rightarrow j)}(\mathbf{x})) \dir \mathbf{x}
\end{align*}
where $U^{(i\rightarrow j)}(\mathbf{x})$ is the same matrix as $U(\mathbf{x})$ from \eqref{eq:u_def} except that the $i$-th row is replaced with the column whose $c$-th entry is $e^{-\frac{1}{4}x_c^2}h_j(x_c)$ for any $\mathbf{x}=(x_0,\dots,x_{\ell-1})\in\mathbb{R}^\ell$.
Furthermore, by Cauchy–Schwarz inequality and comparing to \eqref{eq:w_u_connection},
\begin{align*}
    \abs{\det(W^{(i\rightarrow j)})}
    & \leq
    \left(\int_{x_0>\cdots>x_{\ell-1},\mathbf{x}\in S^{\ell}} (\det(U(\mathbf{x})))^2 \dir \mathbf{x}\right)^{1/2}\left(\int_{x_0>\cdots>x_{\ell-1},\mathbf{x}\in S^{\ell}} (\det(U^{(i\rightarrow j)}(\mathbf{x})))^2 \dir \mathbf{x}\right)^{1/2} \\
    & =
    \left(\frac{\int_{x_0>\cdots>x_{\ell-1},\mathbf{x}\in S^{\ell}} (\det(U^{(i\rightarrow j)}(\mathbf{x})))^2 \dir \mathbf{x}}{\int_{x_0>\cdots>x_{\ell-1},\mathbf{x}\in S^{\ell}} (\det(U(\mathbf{x})))^2 \dir \mathbf{x}}\right)^{1/2}\abs{\det(W)}. \numberthis \label{eq:det_w_ineq}
\end{align*}

By factoring out the term $e^{-\frac{1}{4}x_c^2}$ for each column, we have
\begin{align*}
    \det(U^{(i\rightarrow j)}(\mathbf{x}))
    & =
    e^{-\frac{1}{4}\sum_{c=0}^{\ell-1}x_c^2}\det(P^{(i\rightarrow j)}(\mathbf{x})) \numberthis\label{eq:det_u_eq}
\end{align*}
where $P^{(i\rightarrow j)}(\mathbf{x})$ is the same matrix as $P(\mathbf{x})$ from \eqref{eq:p_def} except that the $i$-th row is replaced with the row whose $c$-th entry is $h_j(x_c)$ for any $\mathbf{x}=(x_0,\dots,x_{\ell-1})\in\mathbb{R}^\ell$.

This time, the computation of $\det(P^{(i\rightarrow j)}(\mathbf{x}))$ is not as easy as $\det(P(\mathbf{x}))$.
In Lemma \ref{lem:det_p_i_j} below, we will show that 
\begin{align*}
    \abs{\det(P^{(i\rightarrow j)}(\mathbf{x}))}
    \leq
    \frac{j!}{i!(\frac{j-i}{2})!}\cdot 2^{O(j)}\cdot \abs{\det(P(\mathbf{x}))}.
\end{align*}

Plugging it into \eqref{eq:det_u_eq} and comparing \eqref{eq:det_u_eq} to \eqref{eq:u_p_connection}, we have
\begin{align*}
    \abs{\det(U^{(i\rightarrow j)}(\mathbf{x}))}
    \leq
    \frac{j!}{i!(\frac{j-i}{2})!}\cdot 2^{O(j)} \cdot \abs{\det(U(\mathbf{x}))}.
\end{align*}
Furthermore, by plugging it into \eqref{eq:det_w_ineq},
\begin{align*}
    \abs{\det(W^{(i\rightarrow j)})}
    &\leq
    \left(\frac{\int_{x_0>\cdots>x_{\ell-1},\mathbf{x}\in S^{\ell}} (\det(U^{(i\rightarrow j)}(\mathbf{x})))^2 \dir \mathbf{x}}{\int_{x_0>\cdots>x_{\ell-1},\mathbf{x}\in S^{\ell}} (\det(U(\mathbf{x})))^2 \dir \mathbf{x}}\right)^{1/2}\abs{\det(W)} 
    \leq
    \frac{j!}{i!(\frac{j-i}{2})!}\cdot2^{O(j)}\cdot\abs{\det(W)}
\end{align*}
Finally, when we plug it into \eqref{eq:det_v_eq}, we prove that
\begin{align*}
    \abs{\det(V^{(i\rightarrow j)})}
    & =
    \frac{i!}{j!}\cdot \frac{\abs{\det(W^{(i\rightarrow j)})}}{\abs{\det(W)}}\cdot \abs{\det(V)} 
     \leq
    \frac{i!}{j!}\cdot \frac{j!2^{O(j)}}{i!(\frac{j-i}{2})!} \cdot \abs{\det(V)} 
     =
    \frac{2^{O(j)}}{(\frac{j-i}{2})!} \cdot \abs{\det(V)}
\end{align*}
Recall that $i\leq 2k-1$ and the assumption of $j\geq \ell > 2(2k-1) \geq 2i$.
We have
\begin{align*}
    \abs{\det(V^{(i\rightarrow j)})} \leq \frac{1}{2^{\Omega(j\log j)}} \cdot \abs{\det(V)}.
\end{align*}

\end{proof}

\begin{lemma}\label{lem:det_v_i_a}

Let $V^{(\mathbf{e}_i\leftarrow a)}$ be the matrix defined in \eqref{eq:v_i_a_def} for $i\leq 2k-1$ and $a\leq \ell$.
Then the absolute value of the  determinant of $V^{(\mathbf{e}_i\leftarrow a)}$ is 
\begin{align*}
    \abs{\det(V^{(\mathbf{e}_i\leftarrow a)})}
    \leq
    2^{O(\ell\log \ell)} \cdot \abs{\det(V)}.
\end{align*}

\end{lemma}

\begin{proof}

Recall that $V^{(\mathbf{e}_i\leftarrow a)}$ is the same matrix as $V$ except that the $a$-th column is replaced with $\mathbf{e}_i$.
Hence, we first expand the determinant along that column and factor out the term $\frac{1}{\sqrt{2\pi}r!}$ for each row.
\begin{align*}
    \abs{\det(V^{(\mathbf{e}_i\leftarrow a)})}
    & =
    \left(\frac{1}{\sqrt{2\pi}}\right)^{\ell-1}\cdot \prod_{r=0,r\neq i}^{\ell-1}\frac{1}{r!}\cdot \abs{\det(W^{(-i,-a)})}
\end{align*}
where $W^{(-i,-a)}$ is the same matrix as $W$ from \eqref{eq:w_def} except that the $i$-th row and the $a$-th column are omitted.
By comparing to \eqref{eq:v_w_connection}, we first simplify $\abs{\det(V^{(\mathbf{e}_i\leftarrow a)})}$ to be
\begin{align*}
    \abs{\det(V^{(\mathbf{e}_i\leftarrow a)})}
    & =
    \sqrt{2\pi}i! \cdot \frac{\abs{\det(W^{(-i,-a)})}}{\abs{\det(W)}}\cdot \abs{\det(V)}
\end{align*}
It means we need to bound the term $\frac{\abs{\det(W^{(-i,-a)})}}{\abs{\det(W)}}$ from above.
To achieve it, we will bound $\abs{\det(W^{(-i,-a)})}$ from above and $\abs{\det(W)}$ from below.

By Cauchy-Binet formula, we further express $\det(W^{(-i,-a)})$ as
\begin{align*}
    \det(W^{(-i,-a)})
    & =
    \int_{x_0>\cdots>x_{\ell-2},\mathbf{x}\in S^{\ell-1}} \det(U^{(-i)}(\mathbf{x}))\det(U^{(-a)}(\mathbf{x})) \dir \mathbf{x} \numberthis\label{eq:det_w_-i_-a}
\end{align*}
where $U^{(-i)}(\mathbf{x})$ (resp. $U^{(-a)}$) is the $(\ell-1)$-by-$(\ell-1)$ matrix whose $(r,c)$-entry is $e^{-\frac{1}{4}x_c^2}h_r(x_c)$  for $r\in[\ell]\backslash\{i\}$ (resp. $r\in[\ell]\backslash\{a\}$), $c\in[\ell-1]$ and any $\mathbf{x}=(x_0,\dots, x_{\ell-2})\in\mathbb{R}^{\ell-1}$.
By factoring out the term $e^{-\frac{1}{4}x_c^2}$ fro each column, 
\begin{align*}
    \det(U^{(-i)}(\mathbf{x})) = e^{-\frac{1}{4}\sum_{c=0}^{\ell-2}x_c^2} \det(P^{(-i)}(\mathbf{x})) \numberthis\label{eq:det_u_-i}
\end{align*}
where $P^{(-i)}(\mathbf{x})$ is the $(\ell-1)$-by-$(\ell-1)$ matrix whose $(r,c)$-entry is $h_r(x_c)$ for $r\in[\ell]\backslash\{i\}$, $c\in[\ell-1]$ and any $\mathbf{x}=(x_0,\dots, x_{\ell-2})\in\mathbb{R}^{\ell-1}$.

Again, the computation of $\det(P^{(-i)}(\mathbf{x}))$ is not as easy as $\det(P(\mathbf{x}))$.
In Lemma \ref{lem:det_p_-i}, we show that 
\begin{align*}
    \abs{\det(P^{(-i)}(\mathbf{x}))}
    \leq
    2^{O(\ell\log \ell)}\cdot  \prod_{1\leq c_1 < c_2 \leq \ell-2}\abs{x_{c_1}-x_{c_2}}.
\end{align*}
Note that the bound is independent to $i$ and hence we have the same bound for $\abs{P^{(-a)}(\mathbf{x})}$.
By plugging it into \eqref{eq:det_u_-i} and further into \eqref{eq:det_w_-i_-a}, we have
\begin{align*}
    \abs{\det(W^{(-i,-a)})}
    & \leq
    2^{O(\ell\log \ell)}\cdot \int_{x_0>\cdots>x_{\ell-2},\mathbf{x}\in S^{\ell-1}} e^{-\frac{1}{2}\sum_{c=0}^{\ell-2}x_c^2}\cdot\prod_{1\leq c_1 < c_2 \leq \ell-2}(x_{c_1}-x_{c_2})^2 \dir \mathbf{x}. \numberthis\label{eq:abs_det_w_-i_-a}
\end{align*}

Recall that, in Lemma \ref{lem:det_v} and \eqref{eq:v_w_connection},
\begin{align*}
    \det(W)
    =
    \int_{x_0>\cdots > x_{\ell-1},\mathbf{x}\in S^{\ell}} e^{-\frac{1}{2}\sum_{c=0}^{\ell-1}x_c^2}\cdot\prod_{0\leq c_1<c_2 \leq \ell-1} (x_{c_1}-x_{c_2})^2 \dir \mathbf{x}.
\end{align*}
Since the term $e^{-\frac{1}{2}\sum_{c=0}^{\ell-1}x_c^2}\cdot\prod_{0\leq c_1<c_2 \leq \ell-1} (x_{c_1}-x_{c_2})^2$ in the integral is symmetric with respect to $x_0,\dots,x_{\ell-1}$,
we have
\begin{align*}
    \det(W)
    =
    \ell!\cdot\int_{\mathbf{x}\in S^{\ell}} e^{-\frac{1}{2}\sum_{c=0}^{\ell-1}x_c^2}\cdot\prod_{0\leq c_1<c_2 \leq \ell-1} (x_{c_1}-x_{c_2})^2 \dir \mathbf{x}.
\end{align*}
To bound $\det(W)$ from below, we consider integrating over the sub-region $\setdef{\mathbf{x}\in S^{\ell}}{\abs{x_{\ell-1}-x_c}>\frac{R}{\ell}}$ of $S^{\ell}$.
\begin{align*}
    \det(W)
    & \geq
    \ell!\cdot\int_{\abs{x_{\ell-1}-x_c}>\frac{R}{\ell},\mathbf{x}\in S^{\ell}} e^{-\frac{1}{2}\sum_{c=0}^{\ell-1}x_c^2}\cdot\prod_{0\leq c_1<c_2 \leq \ell-1} (x_{c_1}-x_{c_2})^2 \dir \mathbf{x} \\
    & \geq
    \ell!\cdot\left(\frac{R}{\ell}\right)^{2(\ell-1)}e^{-\frac{1}{2}R^2}\cdot\int_{\abs{x_{\ell-1}-x_c}>\frac{R}{\ell},\mathbf{x}\in S^{\ell}} e^{-\frac{1}{2}\sum_{c=0}^{\ell-2}x_c^2}\cdot\prod_{0\leq c_1<c_2 \leq \ell-2} (x_{c_1}-x_{c_2})^2 \dir \mathbf{x} \\
    & \geq
    \ell!\cdot R\left(\frac{R}{\ell}\right)^{2(\ell-1)}e^{-\frac{1}{2}R^2}\cdot\int_{\mathbf{x}\in S^{\ell-1}} e^{-\frac{1}{2}\sum_{c=0}^{\ell-2}x_c^2}\cdot\prod_{0\leq c_1<c_2 \leq \ell-2} (x_{c_1}-x_{c_2})^2 \dir \mathbf{x} \\
    & =
    \ell\cdot R\left(\frac{R}{\ell}\right)^{2(\ell-1)}e^{-\frac{1}{2}R^2}\cdot\int_{x_0>\cdots>x_{\ell-2},\mathbf{x}\in S^{\ell-1}} e^{-\frac{1}{2}\sum_{c=0}^{\ell-2}x_c^2}\cdot\prod_{0\leq c_1<c_2 \leq \ell-2} (x_{c_1}-x_{c_2})^2 \dir \mathbf{x} \\
    & =
    \frac{1}{2^{O(\ell\log\ell)}}\cdot\int_{x_0>\cdots>x_{\ell-2},\mathbf{x}\in S^{\ell-1}} e^{-\frac{1}{2}\sum_{c=0}^{\ell-2}x_c^2}\cdot\prod_{0\leq c_1<c_2 \leq \ell-2} (x_{c_1}-x_{c_2})^2 \dir \mathbf{x} \numberthis\label{eq:det_w_lower}
\end{align*}
In other words, by comparing $\abs{\det(W)}$ in \eqref{eq:det_w_lower} to $\abs{\det(W^{(-i,-a)})}$ in \eqref{eq:abs_det_w_-i_-a}, we have
\begin{align*}
    \frac{\abs{\det(W^{(-i,-a)})}}{\abs{\det(W)}}
    & \leq
    2^{O(\ell\log \ell)}
\end{align*}
and hence
\begin{align*}
    \frac{\abs{\det(V^{(\mathbf{e}_i\leftarrow a)})}}{\abs{\det(V)}}
    & =
    \sqrt{2\pi}i!\cdot \frac{\abs{\det(W^{(-i,-a)})}}{\abs{\det(W)}} 
    \leq
    2^{O(\ell\log \ell)}.
\end{align*}

\end{proof}

\begin{lemma} \label{lem:det_p_i_j}

Let $P^{(i\rightarrow j)}(\mathbf{x})$ be the matrix defined in the proof of Lemma \ref{lem:det_v_i_j}.
Then the absolute value of the determinant of $P^{(i\rightarrow j)}(\mathbf{x})$ is 
\begin{align*}
    \abs{\det(P^{(i\rightarrow j)}(\mathbf{x}))}
    \leq
    \frac{j!}{i!(\frac{j-i}{2})!}\cdot 2^{O(j)} \cdot \abs{\det(P(\mathbf{x}))}.
\end{align*}
Recall that $P(\mathbf{x})$ is the matrix defined in \eqref{eq:p_def}.

\end{lemma}

\begin{proof}

Since the entries of $P^{(i\rightarrow j)}(\mathbf{x})$ are Hermite polynomials, we can decompose it into
\begin{align*}
    P^{(i\rightarrow j)}(\mathbf{x})
    =
    C^{(i\rightarrow j)}\cdot X^{[j+1]}
\end{align*}
where $C^{(i\rightarrow j)}$ is the $\ell$-by-$(j+1)$ matrix whose $(r,c)$-entry is the coefficient of $x^c$ in the $r$-th Hermite polynomial and $X^{[j+1]}$ is the $(j+1)$-by-$\ell$ matrix whose $(r,c)$-entry is $x_c^r$.
For example, take $\ell=4,i=2,j=6$, 
\begin{align*}
    h_0(x) & = 1\\
    h_1(x) & = x\\
    h_3(x) & = -3x+x^3 \\
    h_6(x) & = -15+45x^2-15x^4+x^6
\end{align*}
and hence
\begin{align*}
    C^{(i\rightarrow j)}
    & =
    \begin{bmatrix}
    1 & 0 & 0 & 0 & 0 & 0 & 0 \\
    0 & 1 & 0 & 0 & 0 & 0 & 0 \\
    -15 & 0 & 45 & 0 & -15 & 0 & 1 \\
    0 & -3 & 0 & 1 & 0 & 0 & 0
    \end{bmatrix}
\end{align*}

To compute $\det(P^{(i\rightarrow j)}(\mathbf{x}))$, we use Cauchy-Binet formula and we have
\begin{align*}
    \det(P^{(i\rightarrow j)}(\mathbf{x}))
    =\sum_{T} \det(C^{(i\rightarrow j)}_{:,T})\cdot \det(X^{[j+1]}_{T,:})
\end{align*}
where the summation is over all subset $T$ of size $\ell$ of $[j+1]$, $C^{(i\rightarrow j)}_{:,T}$ is the $\ell$-by-$\ell$ matrix whose columns are the columns of $C^{(i\rightarrow j)}$ at indices from $T$ and $X^{[j+1]}_{T,:}$ is the $\ell$-by-$\ell$ matrix whose rows are the rows of $X^{[j+1]}$ at indices from $T$.
Here, for any positive integer $n$, we denote $[n]$ to be the set $\{0,1,\dots,n-1\}$.
Furthermore, by triangle inequality,
\begin{align*}
    \abs{\det(P^{(i\rightarrow j)}(\mathbf{x}))}
    \leq\sum_{T} \abs{\det(C^{(i\rightarrow j)}_{:,T})}\cdot \abs{\det(X^{[j+1]}_{T,:})} \label{eq:det_p_ineq}\numberthis
\end{align*}
We first make some simplifications to see what $T$ makes the determinants nonzero.
For example, take $\ell=8, i=2, j=10$, we have
\begin{align*}
    h_0(x) & = \color{red}{1}\\
    h_1(x) & = \color{blue}{x}\\
    h_3(x) & = \color{blue}{-3x+x^3} \\
    h_4(x) & = \color{red}{3-6x^2+x^4} \\
    h_5(x) & = \color{blue}{15x-10x^3+x^5}\\
    h_6(x) & = \color{red}{-15+45x^2-15x^4+x^6} \\
    h_7(x) & = \color{blue}{-105x+105 x^3-21x^5+x^7}\\
    h_{10}(x) & = \color{red}{-945+4725x^2-3150x^4+630x^6-45x^8+x^{10}}
\end{align*}
and
\begin{align*}
    C^{(i,j)} =_{\text{up to row and column swaps}} 
    \begin{bmatrix}\color{red}{\begin{matrix}
    1 & & & & & \\
    3 & -6 & 1 & & & \\
    -15 & 45 & -15& 1 & & \\
    -945 & 4725 & -3150 & 630 & -45 & 1
     \end{matrix}} & \\
      & \color{blue}{\begin{matrix}
      1 & & & & \\
      -3 & 1 & & & \\
      15 & -10 & 1 & & \\
      -105 & 105 & -21 & 1 &
    \end{matrix}}
    \end{bmatrix}
\end{align*}
For simplicity, we assume that $i,j,\ell$ are even numbers and it is easy to prove the other cases by symmetry.
If $T$ satisfies one of the following conditions:
\begin{itemize}
\item does not contain all odd numbers less than $\ell$, i.e. $1,3,\dots,\ell-1$
\item does not contain all even numbers less than $i$, i.e. $0,2,\dots,i-2$
\item contains more than one even number larger than or equal to $\ell$, i.e. $\ell, \ell+2, \dots, j$
\end{itemize}
then $\det(C_{:,T}^{(i\rightarrow j)}) = 0$.
In other words, the choices are 
\begin{itemize}
    \item $T=[\ell]$ or
    \item $T=[\ell]\backslash\{a\}\cup\{b\}$ for $a=i,i+2,\dots,\ell-2$ and $b=\ell,\ell+2,\dots,j$.
\end{itemize}
Therefore, there are only $\frac{\ell-i}{2}\cdot\frac{j-\ell+2}{2} + 1 = O(j^2)$ choices for $T$ such that $\det(C_{:,T}^{(i\rightarrow j)})$ may not be $0$.

If $T=[\ell]$, by expanding the determinant $\det(C_{:,T}^{(i\rightarrow j)})$ along the rows whose diagonal entry is $1$, what we have left is the determinant of a matrix $A$ where $A$ is the $(\frac{\ell-i}{2})$-by-$(\frac{\ell-i}{2})$ matrix whose $(r,c)$-entry is $(-1)^{\frac{r-c}{2}}\frac{r!}{(\frac{r-c}{2})!c!2^{\frac{r-c}{2}}}$ for $r=i+2,\dots,\ell-2, j$ and $c=i,i+2,\dots,\ell-2$.
In the example, the matrix $A$ is $\begin{bmatrix}
-6 & 1 & \\
45 & -15 & 1 \\
4725& -3150 & 630
\end{bmatrix}$.
By applying row and column operations, we can compute the exact expression for $\det(A)$
\begin{align*}
    \det(A)
    & =
    (-1)^{\frac{j-i}{2}}\frac{j!}{i!2^{\frac{j-i}{2}}}\left(\sum_{m=0}^{\frac{\ell-i-2}{2}} (-1)^m\frac{1}{m!(\frac{j-i}{2}-m)!} \right).
\end{align*}
In the example, we have
\begin{align*}
    \det(\begin{bmatrix}
-6 & 1 & \\
45 & -15 & 1 \\
4725& -3150 & 630
\end{bmatrix})
&=
14175
\end{align*}
Note that the expression $\sum_{m=0}^{\frac{\ell-i-2}{2}} (-1)^m\frac{1}{m!(\frac{j-i}{2}-m)!}$ in  the equation for $\det(A)$ can be easily bounded by
\begin{align*}
    \abs{\sum_{m=0}^{\frac{\ell-i-2}{2}} (-1)^m\frac{1}{m!(\frac{j-i}{2}-m)!}}
    \leq
    \sum_{m=0}^{\frac{\ell-i-2}{2}} \frac{1}{m!(\frac{j-i}{2}-m)!}
    \leq
    \sum_{m=0}^{\frac{j-i}{2}} \frac{1}{m!(\frac{j-i}{2}-m)!}
    =
    \frac{2^{\frac{j-i}{2}}}{(\frac{j-i}{2})!}
\end{align*}
Hence, we have
\begin{align*}
    \abs{\det(C_{:,T}^{(i\rightarrow j)})} = \abs{\det(A)} \leq \frac{j!}{i!(\frac{j-i}{2})!}
\end{align*}
Also, since $T=[\ell]$, therefore $\abs{\det(X^{[j+1]}_{T,:})} =  \prod_{0\leq c_1<c_2 \leq \ell-1} \abs{x_{c_1}-x_{c_2}}$.
When $T=[\ell]$, we have
\begin{align*}
     \abs{\det(C_{:,T}^{(i\rightarrow j)})}\cdot\abs{\det(X^{[j+1]}_{T,:})} \leq \frac{j!}{i!(\frac{j-i}{2})!}\cdot \prod_{0\leq c_1<c_2 \leq \ell-1} \abs{x_{c_1}-x_{c_2}}
\end{align*}

Now, consider the case that $T=[\ell]\backslash\{a\}\cup\{b\}$ for $a=i,i+2,\dots,\ell-2$ and $b=\ell,\ell+2,\dots,j$.
Similar to the previous calculation, by expanding the determinant $\det(C_{:,T}^{(i\rightarrow j)})$ along the rows whose diagonal entry is $1$, what we have left is the determinant of a matrix $A$ where $A$ is the $(\frac{\ell-i}{2})$-by-$(\frac{\ell-i}{2})$ matrix whose $(r,c)$-entry is $(-1)^{\frac{r-c}{2}}\frac{r!}{(\frac{r-c}{2})!c!2^{\frac{r-c}{2}}}$ for $r=i+2,\dots,a, j$ and $c=i,i+2,\dots,a-2,b$.
For example, take $a=6$ and $b=8$, the matrix $A$ is the example is $\begin{bmatrix}
-6 & 1 & \\
45 & -15 & \\
4725 & -3150 & -45
\end{bmatrix}$.
By applying row and column operations, we can compute the exact expression for $\det(A)$ 
\begin{align*}
    \det(A)
    =
    (-1)^{\frac{j-b}{2}}\frac{j!}{(\frac{j-b}{2})!b!2^{\frac{j-b}{2}}} \cdot (-1)^{\frac{a-i}{2}}\frac{a!}{(\frac{a-i}{2})!i!2^{\frac{a-i}{2}}}
\end{align*}
In the example, we have
\begin{align*}
    \det(\begin{bmatrix}
-6 & 1 & \\
45 & -15 & \\
4725 & -3150 & -45
\end{bmatrix})
=
-2025
\end{align*}
To bound $\abs{\det(A)}$,
\begin{align*}
    \abs{\det(A)}
    & =
    \frac{j!}{(\frac{j-b}{2})!b!2^{\frac{j-b}{2}}}\cdot\frac{a!}{(\frac{a-i}{2})!i!2^{\frac{a-i}{2}}}
    =
    \frac{j!}{i!}\cdot \frac{a!}{(\frac{a-i}{2})!b!(\frac{j-b}{2})!} \cdot \frac{1}{2^{\frac{j-b+a-i}{2}}}
\end{align*}
Note that $\frac{1}{2^{\frac{j-b+a-i}{2}}}\leq 1$.
Recall that $i\leq a\leq \ell-2$ and $\ell\leq b\leq j$.
We also have 
\begin{align*}
    \frac{a!}{(\frac{a-i}{2})!} \leq 2^a\cdot (\frac{a+i}{2})! \leq 2^j\cdot(\frac{b+i}{2})!.
\end{align*}
Hence, 
\begin{align*}
    \abs{\det(A)} \leq \frac{j!}{i!}\cdot \frac{2^j(\frac{b+i}{2})!}{b!(\frac{j-b}{2})!}
    =
    \frac{j!}{i!}\cdot 2^j\cdot\frac{(\frac{b+i}{2})!(\frac{b-i}{2})!}{b!}\cdot\frac{(\frac{j-i}{2})!}{(\frac{b-i}{2})!(\frac{j-b}{2})!}\cdot \frac{1}{(\frac{j-i}{2})!}
\end{align*}
Observe that 
\begin{align*}
    \frac{(\frac{b+i}{2})!(\frac{b-i}{2})!}{b!} \leq 1 \qquad \text{and} \qquad  \frac{(\frac{j-i}{2})!}{(\frac{b-i}{2})!(\frac{j-b}{2})!}\leq 2^{\frac{j-i}{2}} \leq 2^{\frac{j}{2}}.
\end{align*}
By plugging them into the above inequality,
\begin{align*}
    \abs{\det(C_{:,T}^{(i\rightarrow j)})}
    =\abs{\det(A)}
    \leq
    \frac{j!}{i!}\cdot \frac{2^{\frac{3j}{2}}}{(\frac{j-i}{2})!}
\end{align*}
Since $a$ is omitted from $\{i,i+2,\dots,\ell-2\}$ and $b$ is selected from $\{\ell,\ell+2,\dots,j\}$, it means that $T=[\ell]\backslash\{a\}\cup\{b\}$.
By the properties of Schur polynomials,
\begin{align*}
\det(X^{[j+1]}_{T,:}) = \left(\sum_Y \mathbf{x}^Y\right)\cdot\prod_{1\leq c_1 < c_2 \leq \ell-1}(x_{c_1}-x_{c_2})
\end{align*}
where the summation is over all semi-standard Young tableaux $Y$ of shape $(b-\ell+1,\underbrace{1,\dots,1}_{\text{$\ell-1-a$ $1$'s}},\underbrace{0,\dots,0}_{\text{$a$ $0$'s}})$.
Here, the term $\mathbf{x}^Y$ means $x_0^{y_0}\cdots x_{\ell-1}^{y_{\ell-1}}$ where $y_m$ is the number of occurrences of the number $m$ in $Y$ and note that $\sum_{m=0}^{\ell-1}y_m = b-a$.
Based on the given shape, there is one row of size $b-\ell-1$ and one column of size $\ell-a$ and they connect at the first element.
For the row, the number of non-decreasing sequences of size $b-\ell-1$ whose numbers are between $0$ and $\ell-1$ inclusive is ${b \choose \ell-1} \leq 2^j$.
For the column, the number of increasing sequences of size $\ell-a$ whose numbers are between $0$ and $\ell-1$ inclusive is ${\ell \choose a} \leq 2^j$.
Hence, the number of semi-standard Young tableaux of such shape is bounded by ${b \choose \ell-1} \cdot {\ell \choose a} \leq 2^{2j}$.
By the assumption that $S=[-R,R]$, we can also bound the term $\abs{\mathbf{x}^Y}$ to be
\begin{align*}
    \abs{\mathbf{x}^Y}
    \leq 
    R^{b-a}
    \leq 
    2^{O(j)}.
\end{align*}
We can now bound the determinant $\abs{\det(X^{[j+1]}_{T,:})}$ by 
\begin{align*}
    \abs{\det(X^{[j+1]}_{T,:})}
    \leq 
    2^{O(j)} \cdot \prod_{1\leq c_1 < c_2 \leq \ell-1}(x_{c_1}-x_{c_2}).
\end{align*}
Namely, when $T=[\ell]\backslash\{a\}\cup\{b\}$ for $a=i,i+2,\dots,\ell-2$ and $b=\ell,\ell+2,\dots,j$, 
\begin{align*}
    \abs{\det(C_{:,T}^{(i\rightarrow j)})}\cdot\abs{\det(X^{[j+1]}_{T,:})}
    & \leq
    \frac{j!}{i!}\cdot \frac{2^{\frac{3j}{2}}}{(\frac{j-i}{2})!} \cdot 2^{O(j)}\cdot \prod_{1\leq c_1 < c_2 \leq \ell-1}(x_{c_1}-x_{c_2}) \\
    & =
    \frac{j!}{i!(\frac{j-i}{2})!}\cdot2^{O(j)}\cdot \prod_{0\leq c_1<c_2 \leq \ell-1} \abs{x_{c_1}-x_{c_2}}
\end{align*}

By considering all cases for $T$ and plugging them into \eqref{eq:det_p_ineq}, we have
\begin{align*}
    \abs{\det(P^{(i\rightarrow j)}(\mathbf{x}))}
    & \leq
    \sum_{T} \abs{\det(C^{(i\rightarrow j)}_{:,T})}\cdot \abs{\det(X^{[j+1]}_{T,:})} 
    \leq
    \frac{j!}{i!(\frac{j-i}{2})!}\cdot 2^{O(j)}\cdot \prod_{0\leq c_1<c_2 \leq \ell-1} \abs{x_{c_1}-x_{c_2}}
\end{align*}
and, by comparing to $\det(P(\mathbf{x}))$ in \eqref{eq:det_p} which is $\prod_{0\leq c_1<c_2 \leq \ell-1} \abs{x_{c_1}-x_{c_2}}$, 
\begin{align*}
    \abs{\det(P^{(i\rightarrow j)}(\mathbf{x}))}
    \leq
    \frac{j!}{i!(\frac{j-i}{2})!}\cdot 2^{O(j)} \cdot \abs{\det(P(\mathbf{x}))}.
\end{align*}

\end{proof}

\begin{lemma}\label{lem:det_p_-i}

Let $P^{(-i)}(\mathbf{x})$ be the matrix defined in the proof of Lemma \ref{lem:det_v_i_a}.
Then the absolute value of the determinant of $P^{(-i)}(\mathbf{x})$ is 
\begin{align*}
    \abs{\det(P^{(-i)}(\mathbf{x}))}
    \leq
    2^{O(\ell\log \ell)}\cdot  \prod_{1\leq c_1 < c_2 \leq \ell-2}\abs{x_{c_1}-x_{c_2}}.
\end{align*}

\end{lemma}

\begin{proof}

Since the entries of $P^{(-i)}(\mathbf{x})$ are Hermite polynomials, we can decompose it into
\begin{align*}
    P^{(-i)}(\mathbf{x})
    =
    C^{(-i)}\cdot X^{[\ell]}
\end{align*}
where $C^{(-i)}$ is the $(\ell-1)$-by-$\ell$ matrix whose $(r,c)$-entry is the coefficient of $x^c$ in the $r$-th Hermite polynomial for $r\in[\ell]\backslash\{i\}$ and $X^{[\ell]}$ is the $\ell$-by-$(\ell-1)$ matrix whose $(r,c)$-entry is $x_c^r$.
For example, take $\ell=4,i=2$, 
\begin{align*}
    h_0(x) & = 1\\
    h_1(x) & = x\\
    h_3(x) & = -3x+x^3 \\
\end{align*}
and hence
\begin{align*}
    C^{(-i)}
    & =
    \begin{bmatrix}
    1 & 0 & 0 & 0\\
    0 & 1 & 0 & 0 \\
    0 & -3 & 0 & 1
    \end{bmatrix}
    \qquad \text{and} \qquad 
    X^{[\ell]}
    =
    \begin{bmatrix}
    1 & 1 & 1 \\
    x_0 & x_1 & x_2 \\
    x_0^2 & x_1^2 & x_2^2 \\
    x_0^3 & x_1^3 & x_2^3 \\
    \end{bmatrix}.
\end{align*}

To compute $\det(P^{(-i)}(\mathbf{x}))$, we use Cauchy-Binet formula and we have
\begin{align*}
    \det(P^{(-i)}(\mathbf{x}))
    =\sum_{T} \det(C^{(-i)}_{:,T})\cdot \det(X^{[\ell]}_{T,:})
\end{align*}
where the summation is over all subset $T$ of size $\ell-1$ of $[\ell]$, $C^{(i\rightarrow j)}_{:,T}$ is the $(\ell-1)$-by-$(\ell-1)$ matrix whose columns are the columns of $C^{(-i)}$ at indices from $T$ and $X^{[\ell]}_{T,:}$ is the $(\ell-1)$-by-$(\ell-1)$ matrix whose rows are the rows of $X^{[\ell]}$ at indices from $T$.
Furthermore, by triangle inequality,
\begin{align*}
    \abs{\det(P^{(-i)}(\mathbf{x}))}
    \leq\sum_{T} \abs{\det(C^{(-i)}_{:,T})}\cdot \abs{\det(X^{[\ell]}_{T,:})} \numberthis\label{eq:det_p_-i}
\end{align*}
We first make some simplifications to see what $T$ makes the determinants nonzero.
For example, take $\ell=8, i=2$, we have
\begin{align*}
    h_0(x) & = \color{red}{1}\\
    h_1(x) & = \color{blue}{x}\\
    h_3(x) & = \color{blue}{-3x+x^3} \\
    h_4(x) & = \color{red}{3-6x^2+x^4} \\
    h_5(x) & = \color{blue}{15x-10x^3+x^5}\\
    h_6(x) & = \color{red}{-15+45x^2-15x^4+x^6} \\
    h_7(x) & = \color{blue}{-105x+105 x^3-21x^5+x^7}\\
\end{align*}
and
\begin{align*}
    C^{(-i)} =_{\text{up to row and column swaps}} 
    \begin{bmatrix}\color{red}{\begin{matrix}
    1 & & &  \\
    3 & -6 & 1 &  \\
    -15 & 45 & -15& 1 
     \end{matrix}} & \\
      & \color{blue}{\begin{matrix}
      1 & & & \\
      -3 & 1 & &  \\
      15 & -10 & 1 & \\
      -105 & 105 & -21 & 1 
    \end{matrix}}
    \end{bmatrix}
\end{align*}
Fro simplicity we assume that $i,\ell$ are even numbers and it is easy to prove the other cases by symmetry.
If $T$ does not contain all odd numbers or all even numbers less than $i$, then $\det(C^{(-i)}_{:,T})=0$.
In the words, the choices are $[\ell]\backslash\{b\}$ for $b=i,i+2,\dots,\ell-2$.
Therefore, there are only $\frac{\ell-i}{2}=O(\ell)$ choices for $T$ such that $\det(C^{(-i)}_{:,T})$ may be be $0$.

Now, we expand the determinant $\det(C^{(-i)}_{:,T})$ along the rows whose diagonal entry is $1$.
What we have left is the determinant of a matrix $A$ where is $A$ is the $(\frac{b-i}{2})$-by-$(\frac{b-i}{2})$ matrix whose $(r,c)$-entry is $(-1)^{\frac{r-c}{2}}\frac{r!}{(\frac{r-c}{2})!c!2^{\frac{r-c}{2}}}$ for $r=i+2,\dots,b$ and $c=i,i+2,\dots,b-2$.
For example, take $b=6$, the matrix $A$ in the above example is $\begin{bmatrix}
-6 & 1 \\
45 & -15
\end{bmatrix}$.
By applying row and column operations, we can compute the exact expression for $\det(A)$ as
\begin{align*}
    \det(A) = (-1)^{\frac{b-i}{2}}\frac{b!}{(\frac{b-i}{2})!i!2^{\frac{b-i}{2}}}
\end{align*}
and hence
\begin{align*}
    \abs{\det(C^{(-i)}_{:,T})}
    =
    \abs{\det(A)}
    & \leq
    \frac{b!}{(\frac{b-i}{2})!i!2^{\frac{b-i}{2}}}
    \leq
    \ell!. \numberthis\label{eq:det_c_-i}
\end{align*}
In the example, we have
\begin{align*}
    \det(\begin{bmatrix}
-6 & 1 \\
45 & -15
\end{bmatrix}) = 45.
\end{align*}

By the properties of Schur polynomials,
\begin{align*}
\det(X^{[\ell]}_{T,:}) = \left(\sum_Y \mathbf{x}^Y\right)\cdot\prod_{1\leq c_1 < c_2 \leq \ell-2}(x_{c_1}-x_{c_2})
\end{align*}
where the summation is over all semi-standard Young tableaux $Y$ of shape $(\underbrace{1,\dots,1}_{\text{$\ell-1-b$ $1$'s}},\underbrace{0,\dots,0}_{\text{$b$ $0$'s}})$.
Recall that the term $\mathbf{x}^Y$ means $x_0^{y_0}\cdots x_{\ell-2}^{y_{\ell-2}}$ where $y_m$ is the number of occurrences of the number $m$ in $Y$ and note that $\sum_{m=0}^{\ell-2}y_m = \ell-1-b$.
Based on the given shape, there is only one column of size $\ell-1-b$.
That means the number of semi-standard Young tableaux of such shape is the number of increasing sequences of size $\ell-1-b$ whose numbers are between $0$ and $\ell-2$ inclusive which is ${\ell-1 \choose b} \leq 2^{\ell}$.
By the assumption that $S=[-R,R]$, we can also bound the term $\abs{\mathbf{x}^Y}$ to be
\begin{align*}
    \abs{\mathbf{x}^Y} \leq R^{\ell-1-b} \leq 2^{O(\ell)}.
\end{align*}
It means that 
\begin{align*}
    \abs{\det(X_{T,:}^{[\ell]})}
    & \leq 
    2^{O(\ell)}\cdot \prod_{1\leq c_1 < c_2 \leq \ell-2}(x_{c_1}-x_{c_2}). \numberthis\label{eq:det_x_-i}
\end{align*}

By plugging \eqref{eq:det_c_-i} and \eqref{eq:det_x_-i} into \eqref{eq:det_p_-i}, we can now bound $\abs{\det(P^{(-i)}(\mathbf{x}))}$ by
\begin{align*}
    \abs{\det(P^{(-i)}(\mathbf{x}))}
    & \leq
    \sum_{T} \abs{\det(C^{(-i)}_{:,T})}\cdot \abs{\det(X^{[\ell]}_{T,:})}
    \leq
    2^{O(\ell\log \ell)}\cdot  \prod_{1\leq c_1 < c_2 \leq \ell-2}\abs{x_{c_1}-x_{c_2}}.
\end{align*}

\end{proof}

\end{document}